%% file: Main.tex
\documentclass{article}

\usepackage[final,nonatbib]{neurips_2021}

\usepackage[utf8]{inputenc} 
\usepackage[T1]{fontenc}    
\usepackage{hyperref}       
\usepackage{url}            
\usepackage{booktabs}       
\usepackage{amsfonts}       
\usepackage{nicefrac}       
\usepackage{microtype}      
\usepackage{xcolor}         
\usepackage{amsmath}
\usepackage{changepage}

\usepackage{amsthm}
\newtheorem{theorem}{Theorem}
\newtheorem{corollary}{Corollary}
\newtheorem{lemma}{Lemma}
\newtheorem{constraint}{Constraint}
\newtheorem{definition}{Definition}
\newtheorem{assumption}{Assumption}
\newtheorem{proposition}{Proposition}
\newtheorem{problem}{Problem}
\newtheorem{remark}{Remark}
\usepackage{bm}
\usepackage{graphicx, wrapfig}
\usepackage{caption}
\usepackage{subcaption}
\hypersetup{hidelinks}
\usepackage{color}
\usepackage{array}
\usepackage{multirow}
\input{math_commands}

\newcommand{\er}{\operatorname{Gerr}}
\newcommand{\gen}{\operatorname{Gen}}

\newcommand{\tr}{\operatorname{tr}}
\newcommand{\bS}{\boldsymbol{S}}
\newcommand{\bJ}{\boldsymbol{J}}
\newcommand{\bV}{\boldsymbol{V}}
\newcommand{\bW}{\boldsymbol{W}}
\newcommand{\bz}{\boldsymbol{z}}
\newcommand{\bX}{\boldsymbol{X}}
\newcommand{\bY}{\boldsymbol{Y}}
\newcommand{\bZ}{\boldsymbol{Z}}
\newcommand{\bG}{\boldsymbol{G}}
\newcommand{\bSig}{\boldsymbol{\Sigma}}
\newcommand{\Diag}{\operatorname{Diag}}
\usepackage{float}
\usepackage[linesnumbered,ruled,vlined]{algorithm2e}

\SetKwInput{KwInput}{Input}                
\SetKwInput{KwOutput}{Output}              
\title{Optimizing Information-theoretical Generalization Bounds via  Anisotropic Noise in SGLD}

\author{%
  Bohan Wang\\
  University of Science \& Technology of China
\normalsize
 \\
Microsoft Research Asia 
\And
Huishuai Zhang
\\
Microsoft Research Asia
\And
Jieyu Zhang
\\
University of Washington
\\
Microsoft Research Asia
\And
Qi Meng
\\
Microsoft Research Asia
\And
Wei Chen\thanks{Corresponding Author}
\\
Microsoft Research Asia
\And
Tie-Yan Liu
\\
Microsoft Research Asia
}
\hypersetup{draft}
\begin{document}

\maketitle

\begin{abstract}
Recently, the information-theoretical framework has been proven to be able to obtain non-vacuous generalization bounds for large models trained by Stochastic Gradient Langevin Dynamics (SGLD) with isotropic noise.  In this paper, we optimize the information-theoretical generalization bound by manipulating the noise structure in SGLD. We  prove that with constraint to guarantee low empirical risk, the optimal noise covariance is the square root of  the expected gradient covariance if  both the prior and the posterior are jointly optimized. This validates that the optimal noise is quite close to  the empirical gradient covariance.  Technically, we develop a new information-theoretical bound that enables such an optimization analysis. We then apply matrix analysis to derive the form of optimal noise covariance. Presented constraint and  results are validated by the empirical observations.

\end{abstract}

\input{intro}

\input{preliminaryconstraints}

\input{infotheoryframework}

\section{Obtain the Optimal Noise Covariance with the Greedy Prior}
\label{sec: a_greedy_approach}

The generalization bound in Theorem \ref{thm: reversed_generalization} depends  on both the prior distribution $P$  and the posterior distribution $Q$ (consequently on $\bSig_{[T]}$). Therefore it requires searching $P$ and $\bSig_{[T]}$ jointly to optimize the generalization error. 

In this section, we  solve \textbf{Problem 1}  with greedily selected priors.
Due to that the square root in the bound Eq. (\ref{eq: second_error_bound}) make the dependency on $\bSig_{[T]}$ even more complex, we optimize a slightly different version of the bound by taking the expectation with respect to $\bS_{\bJ^c}$ into the square root, i.e.,    
\begin{equation}
\label{eq: optimizition_target_2}
{\gen}_T \overset{\triangle}{=}   \mathbb{E}_{\bS_{\bJ},\bV_{[T]},\bJ}\sqrt{\frac{(a_2-a_1)^2}{2}\mathbb{E}_{\bS_{\bJ^c}}\operatorname{KL}\left(P^{\bJ,\bS_{\bJ},\bV_{[T]}}\left\Vert Q^{\bS,\bV_{[T]}}\right.\right)}.
\end{equation}
By Jensen's Inequality, Eq. (\ref{eq: optimizition_target_2}) is still a generalization bound, but allows the expectation with respect to $\bS_{\bJ^c}$ to interact with the KL divergence directly.
One can easily observe that $\gen_T$ is a mapping from $P$ and $Q$ (and thus from $P$ and $\bSig_{[T]}$) to a positive real.
Therefore, we can reformulate Problem \ref{problem: ultimate} mathematically as the following optimization problem $ \textbf{(P1)}$:
\begin{equation*}
   \textbf{(P1)}. \min_{P,\bSig_{[T]}} \gen_T(P, \bSig_{[T]}),\text{ subject to: Constraint \ref{constraint: trace}}.
\end{equation*}
For simplicity, we restrict that the considered noise covariance $\bSig_{[T]}$  only depends on the parameter $\bW$, and is independent of the sample $\bS$. By Lemma~\ref{lem: decomposition_kl}, the KL term in the optimization target $\gen_T$ can be decomposed into
\begin{small}
\begin{equation}
\label{eq: decomposition_problem_2}
 \operatorname{KL}\left(P^{\bJ, \bS_{\bJ},\bV_{[T]}}\left\Vert Q^{\bS,\bV_{[T]}}\right.\right)=\sum_{s=1}^T \mathbb{E}_{P_{s-1}^{\bJ, \bS_{\bJ},\bV_{[s-1]}}} \operatorname{KL}\left(P_{s|(s-1)}^{\bJ,\bS_{\bJ},\bV_{s}}\left\Vert Q_{s|(s-1)}^{\bS,\bV_{s}}\right.\right).
\end{equation}
\end{small}Therefore, when optimizing $\gen_T$ with respect to prior $P$, both $\operatorname{KL}(P_{s|(s-1)}^{\bJ, \bS_{\bJ},\bV_{s}}||Q_{s|(s-1)}^{\bS,\bV_{s}})$ and $P_{i-1}^{\bJ,\bS_{\bJ},\bV_{[i-1]}}$ $(i>s)$ have dependence on $P_{s|(s-1)}^{\bJ,\bS_{\bJ},\bV_{s}}$. Similar to the discussion in Section \ref{sec: difficulty_traditional}, the dependence of $P_{i-1}^{\bJ,\bS_{\bJ},\bV_{[i-1]}}$ with $i>s$ on $P_{s|(s-1)}^{\bJ,\bS_{\bJ},\bV_{s}}$ can be very complex. Therefore, we approximate the optimal $P_{s|(s-1)}^{\bJ,\bS_{\bJ},\bV_{s}}$ by the greedy prior which is defined as follows:
\begin{definition}[Greedy Prior]
\label{definition: greedy_prior}
We say $P^*$ is the optimal prior in the greedy sense, or the greedy prior for brevity, if for any $1\le s\le T$ and any $\bS_{\bJ}$ and $\bV_{[T]}$, ${P^*}^{\bJ,\bS_{\bJ},\bV_{s}}_{s|(s-1)}={P^s}^{\bJ,\bS_{\bJ},\bV_{s}}_{s|(s-1)} $, where $P^s$ is defined as follows:
\begin{equation*}
\label{def: greedy_prior}
    P^{s}\overset{\triangle}{=}\arg \min_{P}\left(\min_{\bSig_{[s]}}\gen_s(P, \bSig_{[s]})\right), \text{ subject to: Constraint \ref{constraint: trace}}.
\end{equation*}
\end{definition}

Intuitively, the conditional probability of $P^*$ of the step $s$ is the optimal one if we only consider the generalization bound for steps up to $s$, and a special case is that the step $T$ conditional probability of $P^*$  agrees with the step $T$ conditional probability of $P^T$, which is the desired optimal prior.  This is why we call $P^*$ "greedy" and use it to approximate the optimal prior.


With the greedy prior, we  characterize the optimal noise covariance by the following theorem.
\begin{theorem}
\label{thm: greedy}
Let the iteration of SGLD with state-dependent noise $Q^{\bS,\bV_{[T]}}$ be given as Eq.(\ref{eq:state-dependent SGLD}). Under Constraint \ref{constraint: trace}, the greedy optimal prior of step $t$ is given by
\begin{small}
\begin{equation*}
    \bW_t=\bW_{t-1}-\eta_t \left(\frac{\vert\bV_t \cap \bJ \vert}{\vert\bV_t \vert}\nabla \mathcal{R}_{\bS_{\bV_t \cap \bJ} }\left(\bW_{t-1}\right)+\frac{\vert\bV_t\cap\bJ^c \vert}{\vert\bV_t\vert }\nabla \mathcal{R}_{\mathcal{D} }\left(\bW_{t-1}\right)\right)    + \mathcal{N}\left(\boldsymbol{0},\bSig^*_{t}(\bW_{t-1})\right),
\end{equation*}
\end{small} while the optimal covariance of noise $\bSig^*_{[T]}$ for $\gen_T$ with the greedy prior is given by $ \bSig^{*}_{t}(\bW)=\lambda_t(\bW)\left(\bSig^{pop}_{\bW}\right)^{\frac{1}{2}}$ $(\forall t\in [T])$, where $\lambda_{t}(\bW)=c_t(\bW)/\tr (\left(\bSig^{pop}_{\bW}\right)^{\frac{1}{2}})$.
\end{theorem}
 As the sample size is large enough, we have $ \bSig^{sd}_{\bS,\bW}\rightarrow\bSig^{pop}_{\bW}$ almost surely, which demonstrates the similarity between the solution of \textbf{Problem 1} and the noise covariance of SGD. Also, by the Law of Large Numbers, $\nabla \mathcal{R}_{\bS_{\bJ}}\approx \nabla \mathcal{R}_{\mathcal{D}}$, and the mean of the greedy optimal prior recovers the mean of the prior used in \cite{negrea2019information} (one can also refer to Eq. (\ref{eq: update_prior}) in this paper for the form).

We briefly state the proof skeleton of Theorem \ref{thm: greedy}, with the proof details deferred to Appendix \ref{appen:6}.
To obtain the final optimal noise covariance in Theorem \ref{thm: greedy}, we need to first derive the greedy prior, i.e., the optimal conditional distribution $P_{s|(s-1)}$ of the prior  of step $s$ in terms of the generalization bound $\gen_s$, which has the form
\begin{equation}
\label{eq: greedy_s}
  \mathbb{E}_{\bS_{\bJ},\bV_{[s]},\bJ}\sqrt{\frac{(a_2-a_1)^2}{2}\mathbb{E}_{\bS_{\bJ^c}}\sum_{t=1}^s \mathbb{E}_{P_{t-1}^{\bJ, \bS_{\bJ},\bV_{[t-1]}}} \operatorname{KL}\left(P_{t|(t-1)}^{\bJ,\bS_{\bJ},\bV_{t}}\left\Vert Q_{t|(t-1)}^{\bS,\bV_{t}}\right.\right)}.
\end{equation}
Typically, solving the optimal $P_{s|(s-1)}$ requires optimizing  $\gen_s$ with respect to all $P_{t|(t-1)}$ $t\in [s]$ and $\bSig_{[s]}$, which is still very complex. However, we can tackle this problem in a rather elegant way. We first investigate the optimal noise covariance in terms of a single KL divergence term in $\gen_s$.

\begin{lemma}
\label{lem: optimal_one_kl_formal}
Under Constraint \ref{constraint: trace}, the optimal noise covariance of the following problem
\begin{equation}
\label{eq: target_single_kl}
    \min_{\bSig_s} \left(\min_{{P}_{s|(s-1)}^{\bJ,\bS_{\bJ},\bV_{s}}}\mathbb{E}_{\bS_{\bJ^c}\sim \mathcal{D}} \operatorname{KL}\left({P}_{s|(s-1)}^{\bJ,\bS_{\bJ},\bV_{s}}\left\Vert Q_{s|(s-1)}^{\bS,\bV_{s}}\right.\right)\right).
\end{equation}

is attained at $ \bSig^{*}_{t}(\bW)$, where $ \bSig^{*}_{t}(\bW)$ is defined as Theorem \ref{thm: greedy}.

\end{lemma}

By Lemma \ref{lem: optimal_one_kl_formal}, the optimal solution of Eq. (\ref{eq: target_single_kl}) doesn't rely on $\bJ,\bS_{\bJ}$, or $\bV_{[s]}$. On the other hand, by Eq. (\ref{eq: greedy_s}), 
$\mathbb{E}_{\bS_{\bJ^c}\sim \mathcal{D}} \operatorname{KL}\left({P}_{s|(s-1)}^{\bJ,\bS_{\bJ},\bV_{s}}\left\Vert Q_{s|(s-1)}^{\bS,\bV_{s}}\right.\right)$ $(\forall \bJ,\bS_{\bJ},\bV_{[s]})$ are the only terms depending on ${P}_{s|(s-1)}^{\bJ,\bS_{\bJ},\bV_{s}}$ and $\bSig_s$. Therefore, the optimal $\bSig_s$ is also $\lambda_t(\bW)\left(\bSig^{pop}_{\bW}\right)^{\frac{1}{2}}$, which is formally stated as the following lemma.
\begin{lemma}
\label{lem: form after equivalence}
The optimal $\bSig_s$ and $P_{s|(s-1)}$ in terms of $\gen_{s}$ are the same as $\bSig^*_s$ and $P^*_{s|(s-1)}$ given by Theorem \ref{thm: greedy}, respectively.
\end{lemma}

With the greedy prior derived, we apply it back to the generalization bound $\gen_{T}$. As $\gen_{T}$ depends on $\bSig_s$ also through $\mathbb{E}_{\bS_{\bJ^c}\sim \mathcal{D}} \operatorname{KL}\left({P}_{s|(s-1)}^{\bJ,\bS_{\bJ},\bV_{s}}\left\Vert Q_{s|(s-1)}^{\bS,\bV_{s}}\right.\right)$, by applying Lemma \ref{lem: optimal_one_kl_formal} again, we derive Theorem \ref{thm: greedy}.

\section{Extension: Optimal Noise Covariance with Fixed Priors}
\label{sec: optimization_fixed_prior}

In existing works \cite{negrea2019information, Neu2021InformationTheoreticGB}, the prior distribution is set to be the SGLD with isotropic noise, which (with the notations in Theorem \ref{thm: reversed_generalization})  is given by 
\begin{small}
\begin{align}
\nonumber
    &\tilde{\mathcal{M}}_{t}(\bW_{t-1},\bS_{\bJ},\bV_t,\bJ)
    \\
\label{eq: update_prior}
    &~~~~~~~~~=\bW_{t-1}-\eta_t \left(\frac{\vert\bV_t \cap \bJ \vert}{\vert\bV_t \vert}\nabla \mathcal{R}_{\bS_{\bV_t \cap \bJ} }\left(\bW_{t-1}\right)+\frac{\vert\bV_t \cap \bJ^c \vert}{\vert\bV_t\vert }\nabla \mathcal{R}_{\bS_{ \bJ} }\left(\bW_{t-1}\right)\right)    + \mathcal{N}\left(\boldsymbol{0},\sigma_t \mathbb{I}_d\right),
\end{align}
\end{small}
\noindent where $\sigma_t>0$ is the noise scale of prior noise covariance. In our latter analysis, we generalize the prior by allowing $\sigma_t$ depend on $\bW_{t-1}$. The formal description of the iteration of the prior is deferred to Appendix \ref{sec: descri_algorithm} for completeness.

Therefore, it is also interesting to see what the optimal noise covariance looks like if the prior is fixed as the one commonly adopted in the existing analyses, e.g., Eq. (\ref{eq: update_prior}). We still set the optimization constraint the same as Section \ref{sec: a_greedy_approach}, but change the optimization target a little to 
\begin{equation*}
    \widetilde{\gen}_T\overset{\triangle}{=}\mathbb{E}_{\bS}\sqrt{\frac{(a_2-a_1)^2}{2}\mathbb{E}_{\bV_{[T]},\bJ}\operatorname{KL}\left(P^{\bJ,\bS_{\bJ},\bV_{[T]}}\left\Vert Q^{\bS,\bV_{[T]}}\right.\right)}.
\end{equation*}
By Jensen's Inequality and Theorem \ref{thm: reversed_generalization}, $\widetilde{\gen}_T$ is still a generalization bound, but allows us to treat the expectation of the KL divergence with respect to $\bV_{[T]}$ and $\bJ$ for given $\bS$ as a whole in optimization. Similar trick is also adopted in \cite{negrea2019information} to obtain the final generalization error of SGLD.  The problem can be mathematically formulated as the following optimization problem $\textbf{(P2)}$:
\begin{equation*}
   \textbf{(P2)} \;\;\;\;\;
   \min_{\bSig_{[T]}} \widetilde{\gen}_T(P, \bSig_{[T]}),\text{ subject to: Constraint \ref{constraint: trace}},
\end{equation*}
where $P$ is given by the update rule Eq. (\ref{eq: update_prior}). As $\bS_{\bJ}$ is obtained by removing only one sample from the dataset $\bS$ and the size of $\bS$ is usually large in practice, it is reasonable to make the assumption that for any fixed $\bV_{[T]}$ and $\bS$, the prior distribution is the same regardless of $\bJ$.
\begin{assumption}
\label{assum: j_invariant}
For any fixed dataset $\bS$ and mini-batches $\bV_{[T]}$, the distribution $P^{\bJ,\bS_{\bJ},\bV_{[T]}}$ is invariant of $\bJ$.
\end{assumption}
We also restrict that $c_t(\bS,\bW) \le d\sigma_t$ in order to guarantee the noise scale of the prior  comparable to that of the posterior. 
The optimal noise covariance with prior fixed can then be characterized by the following theorem.
\begin{theorem}
\label{thm: main_theorem}
Let prior and posterior be defined  as Eq.(\ref{eq: update_prior}) and Eq.(\ref{eq:state-dependent SGLD}), respectively. Let Assumption \ref{assum: j_invariant} hold. Then, the solution of $\textbf{(P2)}$ is given by
\begin{small}
\begin{equation*}
    \bSig^{*}_t(\bS,\bW)= \boldsymbol{O}_{\bS,\bW}^{sd} \Diag (\tilde{\omega}^{\bS,\bW}_1,\cdots,\tilde{\omega}^{\bS,\bW}_d)\left(\boldsymbol{O}_{\bS,\bW}^{sd}\right)^\top,
\end{equation*}
\end{small}
where $\tilde{\omega}^{\bS,\bW}_i\geq 0$ $(i\in [d])$ (the exact form is omitted here) and
$\left(\boldsymbol{O}_{\bS,\bW}^{sd}\right)$ is the orthogonal matrix that diagonalizes $\bSig^{sd}_{\bS,\bW}$ as
    $\bSig^{sd}_{\bS,\bW}=\boldsymbol{O}_{\bS,\bW}^{sd}\Diag (\omega^{\bS,\bW}_1,\cdots,\omega^{\bS,\bW}_d)\left(\boldsymbol{O}_{\bS,\bW}^{sd}\right)^\top.$
Moreover, for any $i\ne j$, $\tilde{\omega}^{\bS,\bW}_i\geq\tilde{\omega}^{\bS,\bW}_j$  if and only if $\omega^{\bS,\bW}_i\geq\omega^{\bS,\bW}_j$.
\end{theorem}

The proof together with the the exact formula of $\tilde{\omega}_i^{\bS,\bW}$ is deferred to Appendix \ref{sec: optimization_fixed_prior}. 
From Theorem \ref{thm: main_theorem}, the optimal point $(\bSig^*_t)_{t=1}^T$  is similar to the empirical gradient covariance matrix $\bSig^{sd}$ in two ways: 1) $\{\bSig^*_t\}_{t=1}^T$ share the same eigenvectors with $\bSig^{sd}$; 2) the corresponding eigenvalues of $\bSig^*_t$ has the same order as $\bSig^{sd}$.
Though the value of $\tilde \omega_{i}^{S, W}$ is not comparable to  $ \omega_{i}^{S, W}$ because  $\Sigma_t^*$ is affected by the prior noise and the  posterior noise covariance, which are freely chosen, it can be shown that the condition number of $\Sigma_t^*$ is smaller than $\Sigma^{sd}_{S,W}$ (please refer to Appendix \ref{appen: condition_number} for details).

Theorem \ref{thm: main_theorem} also reveals an interesting correlation between the noise covariance matrices of the prior and posterior. While the optimal noise covariance is affected by the prior noise covariance, it also depends on the distance  between the means of the prior and the posterior (see Lemma \ref{lem: single_step_form_kl_general} in Appendix \ref{appen: single_step_kl}), which brings the information of empirical gradient covariance into the optimal noise structure. That being said, the optimal posterior noise covariance is biased to the empirical gradient covariance from prior covariance. We note that such analysis can be easily extended to arbitrary priors.

\section{Empirical Verification}
\label{sec: experiments}

In this section, we support our theoretical findings with some experiments. We adopt the setting in \cite{zhu2018anisotropic} where a four-layer neural network with $11330$ parameters is used to conduct the classification task on the Fashion-MNIST except that we use 10000 training samples instead of 1200  used in \cite{zhu2018anisotropic}. We defer detailed settings of the experiments to Appendix \ref{appen: experiment}.

We first verify if the empirical gradient covariance is  far from isotropic Gaussian distribution. We plot the ratios of the 1st eigenvalue to the $500$th largest eigenvalue of the empirical gradient covariance  
along the training trajectory of both SGD and Iso-SGLD in Figure \ref{fig:ratio}.  We can see that throughout the training procedure, the ratios of empirical gradient covariance stay around $10^7$ for both SGD and Iso-SGLD. In constrast, for the isotropic Gaussian, the distribution of eigenvalues follows semi-circle law, and the ratio would be around $1.1$ for dimension $11330$. This demonstrates that the energy of empirical gradient covariance concentrates in a very small subspace, less than $5\%$ of the total $11330$ dimensions. Hence the empirical gradient covariance is highly anisotropic.

We next verify 
our main claim Theorem \ref{thm: greedy} that under Constraint \ref{constraint: trace}, the generalization error for SGLD with the optimal noise covariance is much smaller than that for Iso-SGLD. Specifically, we consider the  SGLD with noise covariance equal to the (scaled) square root of the empirical gradient covariance, coined ``SREC-SGLD''. We compare the training loss and the generalization error of SREC-SGLD and Iso-SGLD under Constraint \ref{constraint: trace} on the noise covariance trace. From Figure \ref{fig: test}, we can see that the generalization error of  SREC-SGLD is smaller than Iso-SGLD for different noise trace scales, which supports Theorem \ref{thm: greedy}. Moreover, if we look at the training loss curves, the SREC-SGLD and Iso-SGLD behave almost the same for the same noise trace scale, which supports Lemma \ref{lem:state-dependent sgld optimization}. 

Finally though our analysis does not cover the generalization bound for SGD, we further empirically show how SREC-SGLD performs in comparison with SGD and the Iso-SGLD in Figure \ref{fig:compare-with-sgd}. We can see that with the same noise trace scale, SREC indeed provides a more accurate characterization of SGD compared to Iso-SGLD.

\begin{figure}
\centering
    \begin{minipage}{0.64\textwidth}
        \centering
        \begin{subfigure}{.45\textwidth}
          \centering
          \includegraphics[width=1.0\textwidth]{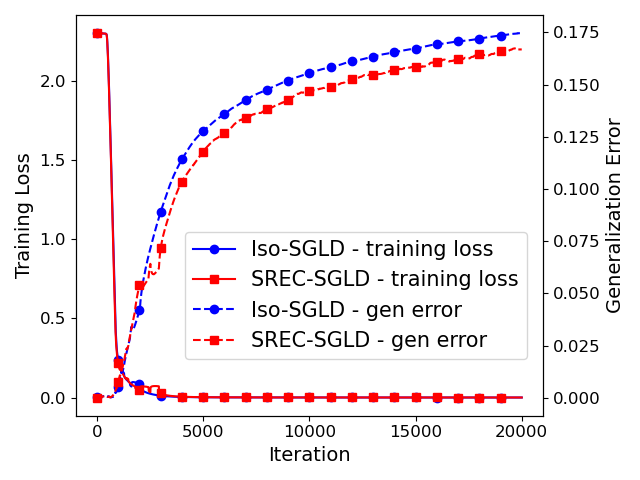}
          \caption{$\text{trace scale}=1$}
          \label{fig:sub1}
        \end{subfigure}%
        \begin{subfigure}{.45\textwidth}
          \centering
          \includegraphics[width=1.0\textwidth]{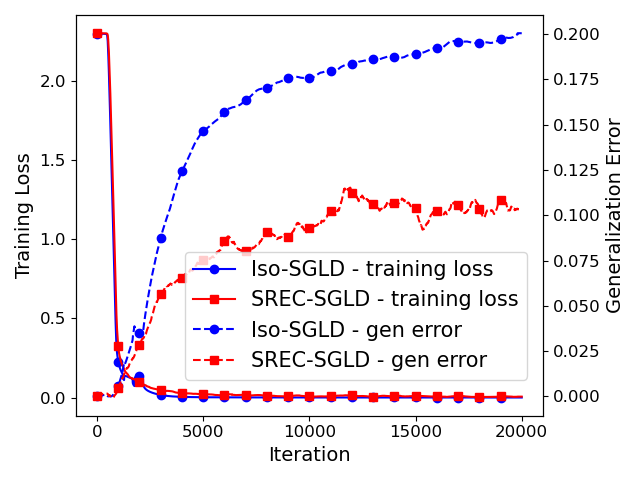}
          \caption{$\text{trace scale}=5$}
          \label{fig:sub2}
        \end{subfigure}
        \caption{The training loss and and the generalization error (test loss $-$ training loss) of the Iso-SGLD and SREC-SGLD. Traces of the noise covariance in (a) and (b) are 1 and 5 times of {\small$\tr((\bSig^{sd}_{\bS,\bW})^{1/2})$}, respectively. 
        }
        \label{fig: test}
    \end{minipage}
    \hfill
    \begin{minipage}{0.32\textwidth}
        \centering    
        \centerline{\includegraphics[width=\columnwidth]{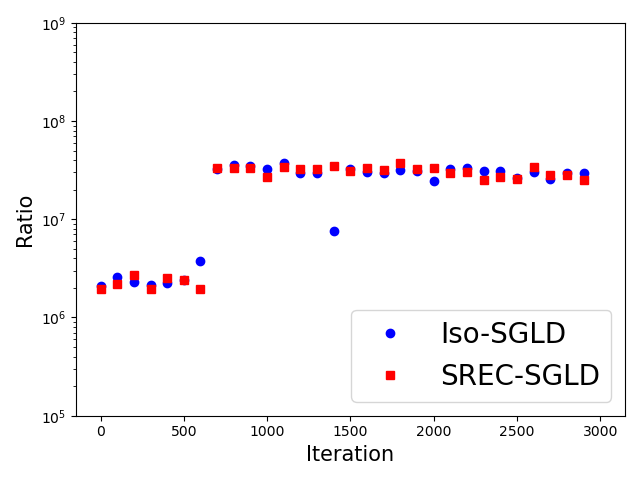}}
        \caption{Ratio of the 1st to the 500th largest eigenvalue of the empirical gradient covariance  for two SGLDs. }
        \label{fig:ratio}
    \end{minipage}
    \vspace{-3mm}
\end{figure}


\section{Conclusion and Future Direction}
\label{sec: conclusion}

In this paper, we study the optimal noise covariance of SGLD in terms of its generalization ability. Specifically, we first formulate the optimization problem both by deriving constraints from the optimization performance and proposing the optimization target by constructing a new information theoretical bound. We then solve the problem with both greedy optimal prior and fixed prior. Interestingly, we observe that the optimal noise covariance aligns with the empirical gradient covariance, which indicates the superiority of the noise covariance of SGD in terms of generalization.

\section*{Acknowledgement}
The authors would like to thank Mr. Ziming Liu for helpful theoretical discussions. 
\newpage
\bibliography{references}
\bibliographystyle{abbrv}

\newpage


\input{Appendix}
\end{document}

%% file: math_commands.tex
\usepackage{amsmath,amsfonts,bm}

\def\eqref#1{equation~\ref{#1}}

\def\1{\bm{1}}

\DeclareMathAlphabet{\mathsfit}{\encodingdefault}{\sfdefault}{m}{sl}
\SetMathAlphabet{\mathsfit}{bold}{\encodingdefault}{\sfdefault}{bx}{n}

\newcommand{\Cov}{\mathrm{Cov}}


%% file: intro.tex
\section{Introduction}
Generalization ability is one of the core questions in learning theory \cite{bousquet2003introduction}, but remains unclear for deep learning models \cite{jiang2019fantastic,zhang2017understanding}. Existing generalization bounds based on the  capacity control become vacuous for practical deep learning models due to the over-parameterization property \cite{arora2018stronger,neyshabur2018pac,zhou2018non}.

From the information theoretical perspective, recent works~\cite{russo2016controlling,xu2017information} bound the generalization error by the mutual information between the dataset and the learned parameters.
This result reveals that good generalization occurs when the learned parameters do not depend on a specific dataset too much, which is intuitively reasonable and closely related with the idea of algorithm stability \cite{charles2018stability,hardt2016train,kuzborskij2018data,mou2018generalization,lei2021sharper} and differential privacy \cite{dwork2015generalization,feldman2018calibrating}. Meanwhile, based on information-theoretic metrics, one can analyze  general classes of updates and models, e.g., stochastic iterative algorithms for non-convex objectives, hence applicable to deep learning. It has been shown that the information theoretical bounds are non-vacuous and  closely related with the real generalization error even in deep learning \cite{haghifam2020sharpened,neyshabur2018pac, dziugaite2017computing,zhou2018non}.

Notably, the information theoretical bound is realized in \cite{pensia2018generalization}  via decomposing the mutual information  across iterations. 
This framework can perform a step-wise analysis of Stochastic Gradient Langevin Dynamics (SGLD) \cite{welling2011bayesian,raginsky2017non}, by evaluating the  
mutual information conditional on the previous learned parameters at each step. We note that the noise added in SGLD is critical for both the mutual information evaluation and the empirical risk minimization.  Most existing work focus on SGLD with constant and isotropic noise covariance \cite{mandt2017stochastic,negrea2019information,haghifam2020sharpened} due to the technical difficulty. However, it is observed that the test accuracy of SGLD with isotropic noise has a considerable gap compared to that of the widely-used Stochastic Gradient Descent (SGD) \cite{zhu2018anisotropic}. This empirical gap motivates us to consider the following question: 
\begin{center}
\emph{Can we find the optimal noise added in SGLD in terms of generalization?}
\end{center}

An affirmative answer will lead to an algorithm imitating SGD better, which helps us to better understand the generalization behavior of SGD.

Specifically, we propose to optimize the structure of the noise in SGLD such that the generalization bound is minimized while a low empirical risk is guaranteed. To this end, we first show that the trace of the noise covariance in SGLD is  a valid constraint that governs  the empirical risk behavior  both theoretically and empirically. Then we devise a new  information theoretical generalization bound that are parallel to those bounds in \cite{negrea2019information}, but facilitate the derivation of the optimal noise. 
With these technical preparations, we  prove that when jointly optimizing both the prior and the posterior, the optimal noise covariance is the square root of  the expected gradient covariance, i.e., their eigenvectors are the same and the corresponding eigenvalues of the former are the square root of the latter,  and the optimal prior recovers the prior in \cite{negrea2019information}.    This indicates the optimal noise covariance of SGLD would  be quite close to the empirical gradient covariance, i.e., the noise covariance of SGD, because of the concentration of measure. 

Our result lends support to the belief that the noise introduced by Stochastic Gradient Descent (SGD) is superior to the isotropic noise, which has been widely observed \cite{keskar2016large,xie2020diffusion,zhu2018anisotropic,zhang2017understanding}. As an illustrative example, we plot the generalization errors of SGD, SGLD with the isotropic noise and SGLD with the optimal noise  in Figure \ref{fig:compare-with-sgd}, where their training curves behave almost the same (do not show here). We can see the optimal noise captures the behavior of SGD  much better than the isotropic noise. 
\begin{wrapfigure}{r}{4.5cm}
\vspace{-15pt}
\includegraphics[width=4.5cm]{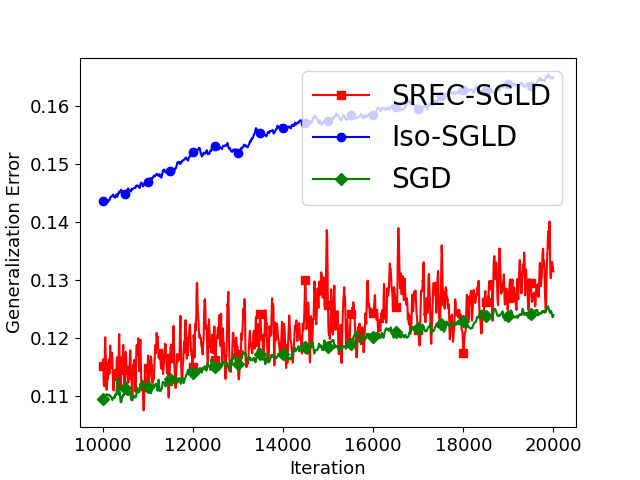}
\vspace{-15pt}
\caption{\small Generalization errors of SGD, SGLD with the isotropic noise (Iso-SGLD) and SGLD with the optimal noise (SREC-SGLD).}
\label{fig:compare-with-sgd}
\vspace{-45pt}
\end{wrapfigure}

Specifically, our contribution can be summarized as follows:
\begin{adjustwidth}{1em}{0em}
 1. We formulate a problem of finding the optimal noise covariance by optimizing an  information-theoretical bound;\\
2. We develop a new information-theoretical bound to facilitate the analysis of the above optimization problem;\\
 3. We obtain the optimal structure of the noise covariances, and demonstrate the similarity to empirical/expected gradient covariance;
\end{adjustwidth}



\subsection{Related Works}
\textbf{Information-theoretical bounds.} Recently, researchers \cite{xu2017information,russo2016controlling,pensia2018generalization} propose to  bound the generalization error by the mutual information between output hypothesis and input samples.  Negrea et al. \cite{negrea2019information} further tighten the bound by designing a data-dependent prior. Following \cite{negrea2019information}, Haghifam et al. \cite{haghifam2020sharpened} obtain comparable results through conditional mutual information \cite{steinke2020reasoning}. Other related work \cite{asadi2018chaining,issa2019strengthened,bu2020tightening,hafez2020conditioning} tightens the information theoretical generalization bounds from different perspectives. There are also high probability generalization bound \cite{mou2018generalization, neyshabur2018pac,li2019generalization} obtained by combining information theory \cite{cover2001elements} and PAC Bayesian framework \cite{parrado2012pac,welling2011bayesian}. 



\textbf{Effects of the noise covariance.} SGD achieves excellent performance in terms of generalization error in deep learning. In contrast, one recent work \cite{zhu2018anisotropic} demonstrates empirically that even with well-tuned scaling,  isotropic SGLD still achieves much worse generalization error than mini-batch SGD. They  obtain an optimal covariance matrix (low-rank approximation of Hessian) in terms of \emph{escaping efficiency} at local minima. Furthermore, \cite{wen2019empirical} show empirically large-batch SGD with diagonal Fisher Gaussian noise can recover similar validation performance as small-batch SGD. However, both the noise covariance investigated by \cite{zhu2018anisotropic} and \cite{wen2019empirical} has gap with that of SGD. It has also been proven \cite{meng2020dynamic} that the stationary distribution of state-dependent SGLD with the empirical loss approximated by quadratic function obeys the power law, and has a better escaping efficiency than  SGLD with the isotropic noise. 





%% file: preliminaryconstraints.tex
\section{Preliminaries} \label{sec:preliminary}

\textbf{Notations.} Here we briefly introduce the notations which will be used throughout this paper.
\begin{itemize}
    \item (Set operation) For a positive integer $N$, we use $[N]$ to denote  the index set $\{1,2,\cdots, N\}$, and use $\bV_{[T]}$ to denote the set  $\{\bV_t\}_{t=1}^T$.  
    For a finite set $\bS=\{\bz_1,\cdots,\bz_N\}$, $\vert \bS\vert$ is the cardinality of $\bS$, and ${\small \bS_{J}\overset{\triangle}{=} \{\bz_i\}_{i\in J}}$ is a subset of $\bS$ with indices in set $J\subset [N]$.
    \item (Probability) $\bz\sim \bS$ denotes that $\bz$ is uniformly sampled from the set $\bS$. For two random variables $\boldsymbol{X}$ and $\boldsymbol{Y}$, we denote the conditional distribution of $\boldsymbol{X}$ given  $\boldsymbol{Y}$ as $\mathbb{P}^{\boldsymbol{Y}}(\boldsymbol{X})$ and denote the conditional expectation of $\boldsymbol{X}$ given $\boldsymbol{Y}$ as $\mathbb{E}^{\boldsymbol{Y}}(\boldsymbol{X})$.
    
    \item (Matrix)  We use $\mathbb{I}_{d\times d}$  for the $d\times d$ identity matrix, abbreviated as $\mathbb{I}$ when dimension is clear. For a differentiable function $f$, we denote the gradient of $f$ at point $\bW$ as $\nabla f(\bW)$. For a positive semi-definite matrix $\boldsymbol{A}\in\mathbb{R}^{d\times d}$, we say $\boldsymbol{B}=\boldsymbol{A}^{\frac{1}{2}}\in\mathbb{R}^{d\times d}$ if $\boldsymbol{B}\cdot \boldsymbol{B}=\boldsymbol{A}$.  
\end{itemize}


\textbf{Supervised Learning. }In this paper, we focus on the supervised learning. It conducts  the empirical risk minimization (ERM) over training data: {\small $    \operatorname{Minimize}_{\bW}  \mathcal{R}_{\bS}(\bW)\overset{\triangle}{=}\frac{1}{\vert \bS \vert} \sum_{\bz\in \bS} \ell (\bz;\bW),$}
where $\bW\in \mathbb{R}^d$ is the parameter of the model, $\bS$ is the training set with each data point i.i.d. sampled from a distribution $\mathcal{D}$, $\ell$ is the (individual) loss function. A \emph{stochastic algorithm} solves the ERM by outputting a  distribution $Q^{\bS}$ of parameter $\bW$. 
The \emph{generalization error} measures the gap  between the population risk and the training risk as
\begin{equation}
    \er\overset{\triangle}{=} \mathbb{E}_{\bS,Q^{\bS}}\left[\mathcal{R}_{\mathcal{D}}(\bW)-\mathcal{R}_{\bS}(\bW)\right]. \label{eq:gerr}
\end{equation}

\textbf{SGD and SGLD. } The update rule of Stochastic Gradient Descent (SGD) at step $t$ is defined as
\begin{equation}
    \bW_{t}\leftarrow\bW_{t-1}-\eta_t\nabla \mathcal{R}_{\bS_{\bV_t}} (\bW_{t-1}), \label{eq:sgd}
\end{equation} 
where $\bV_t$ is sampled uniformly without replacement from $[N]$ with size $b_t$. Given $\bW_{t-1}$, $\bW_t$ is random  with
\begin{small}
\begin{equation*}
\mathbb{E}\left[\bW_{t}|\bW_{t-1} \right]=\bW_{t-1}-\eta_t \nabla \mathcal{R}_{\bS}  \left(\bW_{t-1}\right),  \text{ }   \Cov \left[\bW_{t}|\bW_{t-1} \right]=\frac{N-b_t}{b_t(N-1)}\bSig^{sd}_{\bS,\bW_{t-1}},
\end{equation*}
\end{small}
where  ${\small\bSig^{sd}_{\bS,\bW_{t-1}}\overset{\triangle}{=}\Cov_{\bz\sim \bS}\left[\nabla \mathcal{R}_{\bz}(\bW_{t-1})\right]}$ is the covariance of $\nabla\mathcal{R}_{\bz}$ with $\bz$ single drawn from $\boldsymbol{S}$. The superscript $^{sd}$ means ``single draw''. 
Similarly, we define the population covariance of gradient as $\bSig^{\operatorname{pop}}_{\bW}\overset{\triangle}{=}\Cov_{\bz\sim \mathcal{D}}\left[\nabla \mathcal{R}_{\bz}(\bW)\right]$. We  define the Hessian of $\mathcal{R}_{\bS}(\bW)$ at point $\bW$ as $\mathcal{H}_{\bS,\bW} $.

The Stochastic Gradient Langevine Dynamics (SGLD) is given by 
\begin{equation}
\label{eq:state-dependent SGLD}
    \bW_{t}\leftarrow\bW_{t-1}-\eta_t\nabla \mathcal{R}_{\bS_{\bV_t}} (\bW_{t-1})+\mathcal{N}(\boldsymbol{0},\bSig_t(\bS,\bW_{t-1})),
\end{equation}
where $\bSig_t(\bS,\bW_{t-1})\in \mathbb{R}^{d\times d}$ is a positive semi-definite matrix with dependence on $\bS$ and $\bW_{t-1}$. SGLD (Eq.~\ref{eq:state-dependent SGLD}) with  $\bSig_t(\bS,\bW_{t-1})=c_t \mathbb{I}$ where $c_t$ a positive constant, is called  \emph{isotropic} SGLD, and SGLD with $\bSig_t(\bS,\bW_{t-1})$ dependent on $\bW_{t-1}$ is called \emph{state-dependent} SGLD.

\textbf{Statistics for iterative algorithms. }Both SGD and SGLD are \emph{stochastic iterative algorithms} \cite{gelfand1991recursive}. Specifically, the update rule at step $t$ of a stochastic iterative algorithm $\mathcal{A}$ can be generally characterized as 
    $\bW_t = \mathcal{M}_t (\bW_{t-1}, \bS, \bV_t)$,
where $\bV_t$ is the \emph{auxiliary random variable} at step $t$, e.g., $\bV_t$ in Eq.(\ref{eq:sgd}) and Eq.(\ref{eq:state-dependent SGLD}).  
We use $Q^{\bS,\bV_{[T]}}$ to denote   the joint distribution of $(\bW_t)_{t=0}^T$ conditional on $\{\bS,\bV_{[T]}\}$, $Q_{i:j}^{\bS,\bV_{[T]}}$ to denote the joint distribution of  $(\bW_{t})_{t=i}^j$ conditional on $\{\bS,\bV_{[T]}\}$, and $Q_{j|j-1}^{\bS,\bV_{[T]}}$ to denote the  distribution of $\bW_j$ conditional on $\{\bW_{j-1},\bS,\bV_{[T]}\}$. 

\textbf{Decomposition of KL divergence}. The following Lemma is extensively used throughout this paper.
\begin{lemma}[Proposition 2.6, \cite{negrea2019information}]
\label{lem: decomposition_kl}
Let $Q_{0:T}$ and $P_{0:T}$ are two probability measures on $\mathbb{R}^{d\times (T+1)}$ with $Q_0=P_0$. Then, the KL divergence between $Q_{0:T}$ and $P_{0:T}$ can be decomposed into
\begin{small}
\begin{equation*}
    \operatorname{KL}(Q_{0:T}||P_{0:T})= \sum_{t=1}^{T} \mathbb{E}_{Q_{0:t-1}}\left[\operatorname{KL}\left(Q_{t|[t-1]} \| P_{t|[t-1]}\right)\right].
\end{equation*}
\end{small}
\end{lemma}


\textbf{Information theoretical generalization bound.} Several existing information-theoretical bounds~\cite{haghifam2020sharpened,negrea2019information} share similar framework. We state one representative proposed in \cite{negrea2019information}.   

\begin{proposition}[Theorem 2.5 in \cite{negrea2019information}]
\label{prop: negrea_2019}
Let $\bS$ be the data set i.i.d. sampled from $\mathcal{D}$, and let the loss function $\ell$ be $[a_1,a_2]$ bounded. Let $\mathcal{A}$ be an  algorithm with update rule 
\begin{equation*}
\bW_t\leftarrow \mathcal{M}_t (\bW_{t-1},\bS,\bV_t), \bW_0\sim \mathcal{W}_0   .
\end{equation*} Let $\mathcal{B}$ be another stochastic iterative optimization algorithm with update rule 
\begin{equation*}
    \bW_t\leftarrow\tilde{\mathcal{M}}_t (\bW_{t-1},\bS_{\bJ},\bV_t,\bJ), \bW_0\sim \mathcal{W}_0   .
\end{equation*} where $\bJ$ is sampled uniformly without replacement from $[N]$ with size $N-1$.  Given $\bS$, $\bJ$, and $\bV_{[T]}$, denote the joint distribution of $(\bW_t)_{t=1}^T$ of $\mathcal{A}$  as $Q^{\bS,\bV_{[T]}}$, and the joint distribution of $(\bW_t)_{t=1}^T$ of $\mathcal{B}$ as $P^{\bJ, \bS_{\bJ},\bV_{[T]}}$. Then, for any $\tilde{\mathcal{M}}_{[T]}$, the  generalization error of $\mathcal{A}$ can be bounded as: 
\begin{small}
\begin{align}
\nonumber
  \mathbb{E}_{\bS,\bV_{[T]}}\left[\mathcal{R}_{\mathcal{D}}\left(Q_T^{\bS,\bV_{[T]}}\right)-\hat{\mathcal{R}}_S\left(Q_T^{\bS,\bV_{[T]}}\right)\right]\le \mathbb{E}_{\bS,\bV_{[T]},\bJ }\sqrt{\frac{(a_2-a_1)^2}{2}\operatorname{KL}\left(Q^{\bS,\bV_{[T]}}\left\Vert P^{\bJ,\bS_{\bJ},\bV_{[T]}}\right.\right)}. \label{eq: negrea}
\end{align}
\end{small}
\end{proposition}

This proposition is used to obtain the generalization error bound for SGLD  \cite[Theorem 3.1]{negrea2019information} by further combing Lemma \ref{lem: decomposition_kl}. As $\tilde{\mathcal{M}}_{[T]}$ can be arbitrarily picked, $P$ works as an "Auxiliary Line" and is called the \textbf{prior} distribution, while $Q$ is the real distribution of parameters  called the \textbf{posterior} distribution. In Section \ref{sec: difficulty_traditional}, we will argue the difficulty of applying Proposition \ref{prop: negrea_2019} to analyze the effect of noise structures.


\section{Formulate the Problem: Proper Constraints and Optimization Target}
In this part, we formulate the optimization problem, i.e., finding the optimal noise covariance of SGLD in terms of the information-theoretical generalization bound, by selecting the proper optimization constraint and optimization target. Specifically, in Section \ref{sec: obtaining_constraint}, we argue that the trace of the noise covariance is a proper constraint to ensure the same optimization error; in Section \ref{sec: general_information_theoretical_generalization_bound}, we argue that existing generalization bounds are not proper candidates as the optimization target, and propose a new information-theoretical bound parallel to existing ones but easier to analyze.

\subsection{Constraint on the Covariance to Control the Empirical Risk}
\label{sec: obtaining_constraint}
We first derive the constraint of noise covariance $\bSig_t$ in Eq. (\ref{eq:state-dependent SGLD}) from the perspective of training performance, under which we  optimize the generalization bounds in the rest of this paper. 

Without any constraint on the noise covariance, optimizing generalization error is trivial but meaningless: a direct combination of Theorem 1 of \cite{xu2017information} and Theorem 1 of \cite{pensia2018generalization} shows that for isotropic SGLD with $\bSig_t= \sigma_t \mathbb{I}$, the generalization error after $T$ step satisfies     $\er\le\mathcal{O}(( \sum_{t=1}^T\log (1+1/\sigma_t))^{1/2})$. 
Therefore, as $\sigma_t \rightarrow \infty$ for $t\in[T]$, we have $\er \rightarrow 0$, but then the update of SGLD is dominated by the  noise, leading to arbitrary bad empirical risk. Hence, we need constraints on the covariance in order to control the empirical risk when minimizing the generalization error. Specifically, the expected decrease of the empirical risk for one iteration can be bounded as follows. 

\begin{lemma}
\label{lem:state-dependent sgld optimization}
Let empirical risk $\mathcal{R}_{\bS}(\bW)$ be $\beta$-smooth w.r.t. $\bW$. Let $\bW_{[T]}$ be given by state-dependent SGLD (Eq. (\ref{eq:state-dependent SGLD})). Then,
\begin{scriptsize}
\begin{equation*}
    \mathbb{E}_{t+1|t}\mathcal{R}_{\bS} (\bW_{t+1})-\mathcal{R}_{\bS} (\bW_{t})\le -\left(1-\frac{\beta\eta_{t+1}}{2}\right)\eta_{t+1}\left\|\nabla \mathcal{R}_{\bS}\left(\boldsymbol{W}_{t}\right)\right\|^{2}+\frac{\beta}{2}\tr \left( \frac{\eta_{t+1}^2(N-b_t)}{(N-1)b_t}\bSig^{sd}_{\bS,\bW_t}+\bSig_{t+1}(\bS,\bW_t)\right).
\end{equation*}
\end{scriptsize}
\end{lemma}

The proof can be obtained by a standard analysis in optimization, and we defer it to Appendix \ref{appendix: proof of lemma optimization}. By Lemma \ref{lem:state-dependent sgld optimization}, the noise covariance $\bSig_t$ affects the upper bound of the empirical risk by its \emph{trace}. Therefore, it is reasonable to keep $\operatorname{tr} (\bSig_t(\bS,\bW_{t-1}))$ unchanged while seeking the optimal $\bSig_t(\bS,\bW_{t-1})$ to minimize the generalization error. The constraint is given formally as follows:
\begin{constraint}
\label{constraint: trace}
The trace of $\bSig_t(\bS,\bW_{t-1})$ is fixed when optimizing the generalization error. That is, there exist  positive constants $c_t(\bS,\bW_{t-1})$ depending on $\bS$ and $\bW_{t-1}$, such that,
\begin{equation*}
    \operatorname{tr}(\bSig_t(\bS,\bW_{t-1}))=c_t(\bS,\bW_{t-1}).
\end{equation*}
\end{constraint}

We do not put any constraint on the value of $c_t$ in our latter analyses. Therefore, it is also possible to manipulate $c_t$ in order to jointly optimize the empirical risk and the generalization bound, which however,  is beyond the scope of this paper and we defer it to  future works. Similar constraint is also proposed by \cite{zhu2018anisotropic} from the standpoint of kinetic energy when analyzing the effect of noise structure on the escaping efficiency from saddle points. 




We next verify this constraint empirically. We run SGLD with different covariances on a four-layer neural network for the Fashion-MNIST classification problem. Concretely, the noise covariances are chosen respectively as $\bSig^{(1)}_t=\bSig^{sd}_{\bS,\bW_{t-1}}$, i.e.,  ``EC-SGLD'' ( Empirical Covariance SGLD) curve in Fig.\ref{fig:train_all},  and  $\bSig^{(2)}_t=$ $\frac{1}{d}\tr(\bSig^{sd}_{\bS,\bW_{t-1}}) \mathbb{I}$, i.e., the ``Iso-SGLD (C)'' curve in Fig.\ref{fig:train_all}. It is easy to verify that $\tr(\bSig^{(1)}_t)=\tr(\bSig^{(2)}_t)$, which is exactly the Constraint \ref{constraint: trace}. We can see from Fig.\ref{fig:train_all} that the convergence curves corresponding to $\bSig^{(1)}_t$ and $\bSig^{(2)}_t$ almost coincide with each other, validating the Constraint \ref{constraint: trace}. 
\begin{wrapfigure}{r}{4cm}
\vspace{-5pt}
\includegraphics[width=4cm]{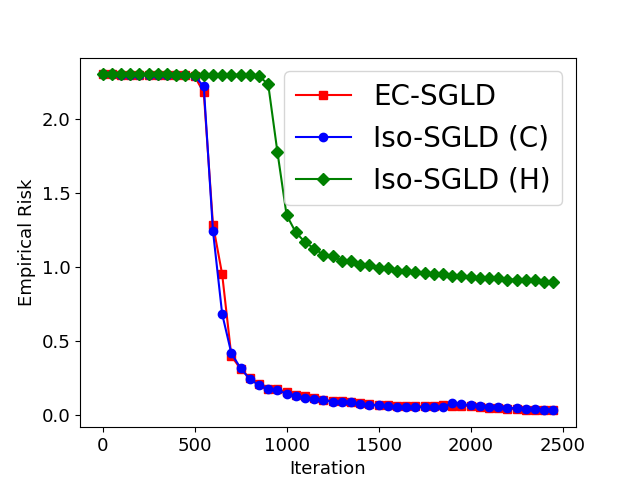}
\caption{\small Training errors of SGLD with different noise covariances. The experiment is run on the Fashion-MNIST dataset with a four-layer neural network (see Appendix \ref{appen: experiment}).}
\label{fig:train_all}
\vspace{-45pt}
\end{wrapfigure}

In comparison, another quantity  $\tr(\mathcal{H}_{\bS,\bW_{t-1}}\bSig_t)$, the trace of the empirical risk's Hessian times the empirical covariance, has also been proposed to govern the convergence behavior \cite[Theorem 4.1]{wen2019empirical}. We plot the convergence curve of SGLD with noise covariance  $\bSig^{(3)}_t=\frac{\tr(\mathcal{H}_{\bS,\bW_{t-1}}\bSig^{sd}_{\bS,\bW_{t-1}})}{\tr(\mathcal{H}_{\bS,\bW_{t-1}})} \mathbb{I}$ ("Iso-SGLD (H)" in Fig.\ref{fig:train_all}). While $\tr(\mathcal{H}_{\bS,\bW_{t-1}}\bSig^{(1)}_t)=\tr(\mathcal{H}_{\bS,\bW_{t-1}}\bSig^{(3)}_t)$, there is a significant gap between curves of ``EC-SGLD '' and ``Iso-SGLD (H)'' in Fig.\ref{fig:train_all}. This implies that $\tr(\mathcal{H}_{\bS,\bW_{t-1}}\bSig_t)$ is not a good constraint for the empirical risk of SGLD, validating the Constraint \ref{constraint: trace} from another side. 

 In the rest of this paper, we will optimize the information-theoretical bound under Constraint \ref{constraint: trace} to ensure low empirical risk.

%% file: infotheoryframework.tex
\subsection{New Information-theoretical Bounds as the Optimization Target}
\label{sec: general_information_theoretical_generalization_bound}
We first demonstrate that existing information-theoretical bounds are not suitable for the optimization target in Section \ref{sec: difficulty_traditional}. Then we propose a new information-theoretical bound as the optimization target in Section \ref{sec: new_bound}.

\subsubsection{Difficulties When Applying Traditional Information-theoretical Bounds}
\label{sec: difficulty_traditional}

By Lemma \ref{lem: decomposition_kl}, the generalization error bound in Proposition \ref{prop: negrea_2019} can be rewritten as 
\begin{equation}
\mathbb{E}_{\bS,\bV_{[T]},\bJ}\sqrt{\frac{(a_2-a_1)^2}{2}\sum_{s=1}^T\mathbb{E}_{Q_{s-1}^{\bS,\bV_{[s-1]}}}\operatorname{KL}\left(Q_{s|(s-1)}^{\bS,\bV_{s}}\left\Vert P_{s|(s-1)}^{\bJ,\bS_{\bJ},\bV_{s}}\right.\right)}.  \label{eq: negrea2} 
\end{equation}


For \textbf{Problem} \ref{problem: ultimate}, there are two difficulties. For a fixed $Q^{\bS,\bV_{[T]}}$, in order to obtain the optimal {\small $P_{s|(s-1)}^{\bJ,\bS_{\bJ},\bV_{s}}$} for any $s\in [T]$, one would first calculate the outside expectation of Eq.~(\ref{eq: negrea2}) over $\bS_{\bJ^c}$ , as $P^{\bS_{\bJ},\bV_{[T]}}$ is independent of $\bS_{\bJ^c}$. However, for each term {\small $\mathbb{E}_{Q_{s-1}^{\bS,\bV_{[s-1]}}}\operatorname{KL} $ $(Q_{s|(s-1)}^{\bS,\bV_{s}}\Vert P_{s|(s-1)}^{\bJ,\bS_{\bJ},\bV_{s}})$ where $s\in[T]$}, both the probability measure {\small$Q_{s-1}^{\bS,\bV_{[s-1]}}$} and the function {\small $\operatorname{KL}(Q_{s|(s-1)}^{\bS,\bV_{s}}\Vert P_{s|(s-1)}^{\bJ,\bS_{\bJ},\bV_{s}})$} has dependency on $\bS_{\bJ^c}$ , which makes evaluating the generalization bound with respect to  $P^{\bS_{\bJ},\bV_{[T]}}$ extremely hard. Furthermore,  when we come to optimize the bound w.r.t. $\bSig_{[T]}$, for each $\bSig_i$ where $i\in[T-1]$,  the Eq.~(\ref{eq: negrea2}) depends on $\bSig_i$  via two terms $A$ and $B$, where
\begin{equation*}
   A= \mathbb{E}_{Q_{i-1}^{\bS,\bV_{[T]}}}\operatorname{KL}\left(Q_{i|(i-1)}^{\bS,\bV_{[T]}}\Vert P_{i|(i-1)}^{\bJ,\bS_{\bJ},\bV_{[T]}}\right),~B=\sum_{s=i+1}^T\mathbb{E}_{Q_{s-1}^{\bS,\bV_{[s-1]}}}\operatorname{KL}\left(Q_{s|(s-1)}^{\bS,\bV_{s}}\Vert P_{s|(s-1)}^{\bJ,\bS_{\bJ},\bV_{s}}\right).
\end{equation*} While the first term is easy to deal with as $Q_{i-1}^{\bS,\bV_{[T]}}$ is irrelevant of $\bSig_i$, the second term depends on $\bSig_i$ via the distribution $Q_{s-1}^{\bS,\bV_{[s-1]}}$ for $s\ge i+1$, and can be very complex (please refer to Appendix \ref{subsec: appen_example}).


\subsubsection{Information-theoretical Generalization Bound for Identifying the Noise Effect}
\label{sec: new_bound}
We then establish a new information-theoretical generalization bound as the optimization target that are parallel to the results in \cite{negrea2019information} but more suitable for analyzing the effect of noise structure. The basic idea  is to reverse the order of prior and posterior in the KL divergence to let $\bSig_t$ only affect one term in the generalization bound for each $t$ when finding the optimal point. The formal theorem is stated as follows:

\begin{theorem}
\label{thm: reversed_generalization}
Let sample set $\bS$, mini-batch $\bV_{[T]}$, random subset $\bJ$, the posterior distribution $Q$ output by Algorithm $\mathcal{A}$ with update rule $\mathcal{M}$  and the prior distribution $P$ output by Algorithm $\mathcal{B}$ with update rule $\tilde{\mathcal{M}}$ be defined as Proposition \ref{prop: negrea_2019}. Let $[a_1,a_2]$ be the range of loss. Then, the  generalization error of $\mathcal{A}$ can be bounded as
\begin{small}
\begin{align}
\label{eq: second_error_bound}
     \er\le\mathbb{E}_{\bS,\bV_{[T]},\bJ}\sqrt{\frac{(a_2-a_1)^2}{2}\operatorname{KL}\left(P^{\bJ,\bS_{\bJ},\bV_{[T]}}\left\Vert Q^{\bS,\bV_{[T]}}\right.\right)}.
\end{align}
\end{small}
\end{theorem}

Compared to Proposition \ref{prop: negrea_2019}, Theorem \ref{thm: reversed_generalization} can be viewed as a parallel version with the positions of the prior distribution $P$ and the posterior distribution $Q$ reversed. This reverse benefits the optimization of the generalization error bound with respect to the noise covariance under Constraint \ref{constraint: trace} in two ways. First, by Lemma \ref{lem: decomposition_kl}, the KL term in Eq.~(\ref{eq: second_error_bound}) can be further decomposed into
\begin{equation*}
   \operatorname{KL}\left(P^{\bS_{\bJ},\bV_{[T]}}\left\Vert Q^{\bS,\bV_{[T]}}\right.\right)= \sum_{s=1}^T\mathbb{E}_{P_{1:s-1}^{\bJ, \bS_{\bJ},\bV_{[s-1]}}}\operatorname{KL}\left(P_{s|(s-1)}^{\bJ,\bS_{\bJ},\bV_{s}}||Q_{s|(s-1)}^{\bS,\bV_{s}}\right),
\end{equation*} 
in which only each KL term depends on each $\bSig_s$, respectively. Secondly, for each summand in the above decomposition, only  $Q_{s|(s-1)}^{\bS,\bV_{s}}$ depends on $\bS_{\bJ^c}$, making it easy to compute the optimal prior in a "greedy" sense (see details in Section \ref{sec: a_greedy_approach}). 

Hence in the rest of this paper, we study the optimal structure of $\bSig_t$ in terms of the generalization bound in Theorem \ref{thm: reversed_generalization} under the Constraint \ref{constraint: trace}. Specifically, our ultimate goal can be stated as follows: 


\begin{problem}
\label{problem: ultimate}
What is the optimal structure of the noise covariance $\bSig_{[T]}$ for SGLD in terms of the generalization bound in Theorem \ref{thm: reversed_generalization} under the Constraint \ref{constraint: trace}? 
\end{problem} 

In the rest of this paper, we focus on solving Problem \ref{problem: ultimate} and its variants.

%% file: Appendix.tex
\appendix

\newpage
\begin{center}
    {\Large \textbf{Supplementary materials for \\``Optimizing Information-theoretical Generalization Bound via  Anisotropic Noise in SGLD''}}
\end{center}

The supplementary materials are organized as follows. In Appendix \ref{appen: prelimi}, we provide some basic lemmas which are used throughout the proofs in the rest of the materials. In Appendix \ref{appendix: proof of lemma optimization}, we provide the proof of Lemma \ref{lem:state-dependent sgld optimization}. In Appendix \ref{appen:4}, \ref{appen:6}, and \ref{appen:5}, we provide the detailed proofs of Lemmas and Theorems respectively in Section \ref{sec: general_information_theoretical_generalization_bound}, Section \ref{sec: a_greedy_approach} , and Section \ref{sec: optimization_fixed_prior}. In Appendix \ref{appen: experiment}, we provide the detailed settings of the experiments in the main text together with an additional experiments to justify the result of Theorem \ref{thm: greedy}.

\section{Preliminaries}
\label{appen: prelimi}
In this section, we provide some basic lemmas that will be used throughout the proof both from probability theory and from matrix analysis.

\subsection{Preparations in Probability Theory}
The first lemma is a standard result characterizing the KL divergence between two Gaussian distributions.
\begin{lemma}[KL divergence between Gaussian distributions]
\label{lem: kl_gau}
Let $\mathrm{P}_1$ and $\mathrm{P}_2$ are multivariate Gaussian distributions on $\mathbb{R}^d$ with mean and covariance respectively $\mu_1$, $\bSig_1$ and $\mu_2$, $\bSig_2$. Then the KL divergence between $\mathrm{P}_1$ and $\mathrm{P}_2$ are given as follows:
\begin{equation*}
    \mathrm{KL}(\mathrm{P}_1||\mathrm{P}_2)=\frac{1}{2}\left(\operatorname{tr}\left(\bSig_{2}^{-1} \bSig_{1}\right)+\left(\mu_{2}-\mu_{1}\right)^{\top} \bSig_{2}^{-1}\left(\mu_{2}-\mu_{1}\right)-d+\ln \left(\frac{\operatorname{det} \bSig_{2}}{\operatorname{det} \bSig_{1}}\right)\right).
\end{equation*}
 \end{lemma}

We then provide a lemma which gives the expected  difference between two uniform sampling variables.
\begin{lemma}[Two step sampling]
\label{lem: sampling}
Suppose $\bz$ is a discrete random variable with $\mathbb{P}(\bz=\bz_i)=\frac{1}{N}$, $\forall i=1,2,\cdots, N$, where the support set is $\mathcal{Z}=\{\bz_1,\cdots,\bz_N\}\subset \mathbb{R}^d$.  Suppose further $\boldsymbol{U}$ is a random index set with size $b$ and sampled uniformly without replacement from $[N]$. Suppose $\bV$ is another random index set independent of $\boldsymbol{U}$ with size $N-1$ and sampled uniformly without replacement from $[N]$. Denote subset of $\mathcal{Z}$ with index in $\boldsymbol{U}\cap \bV^c$ and $\bV$ respectively as $\mathcal{Z}_{\boldsymbol{U}\cap \bV^c}=\{\bz_i, i\in \boldsymbol{U}\cap \bV^c\}$,  $\mathcal{Z}_{\bV}=\{\bz_i, i\in \bV\}$, and the average of $\mathcal{Z}_{\boldsymbol{U}\cap \bV^c}$ and $\mathcal{Z}_{\bV}$ respectively as $\bar{\mathcal{Z}}_{\boldsymbol{U}\cap \bV^c}$ and $\bar{\mathcal{Z}}_{\bV}$. Then the following equation holds:
\begin{equation*}
    \mathbb{E}_{\boldsymbol{U},\bV}\left(\frac{(b-\vert\boldsymbol{U}\cap\bV\vert)^2}{b^2}(\bar{\mathcal{Z}}_{\bV}-\bar{\mathcal{Z}}_{\boldsymbol{U}\cap \bV^c})(\bar{\mathcal{Z}}_{\bV}-\bar{\mathcal{Z}}_{\boldsymbol{U}\cap \bV^c})^\top\right)=\frac{1}{Nb}\left(\frac{N}{N-1}\right)^2 \Cov(\bz).
\end{equation*}
\end{lemma}
\begin{proof}
We rewrite $ \mathbb{E}_{\boldsymbol{U},\bV}\left(\frac{(b-\vert\boldsymbol{U}\cap\bV\vert)^2}{b^2}(\bar{\mathcal{Z}}_{\bV}-\bar{\mathcal{Z}}_{\boldsymbol{U}\cap \bV^c})(\bar{\mathcal{Z}}_{\bV}-\bar{\mathcal{Z}}_{\boldsymbol{U}\cap \bV^c})^\top\right)$ by taking conditional expectation with respect to $\vert \boldsymbol{U}\cap \bV\vert$ as follows:
\begin{align*}
    &\mathbb{E}_{\boldsymbol{U},\bV}\left(\frac{(b-\vert\boldsymbol{U}\cap\bV\vert)^2}{b^2}(\bar{\mathcal{Z}}_{\bV}-\bar{\mathcal{Z}}_{\boldsymbol{U}\cap \bV^c})(\bar{\mathcal{Z}}_{\bV}-\bar{\mathcal{Z}}_{\boldsymbol{U}\cap \bV^c})^\top\right)
    \\
    &=\mathbb{E}_{\vert \boldsymbol{U}\cap \bV^c\vert}\mathbb{E}^{\vert \boldsymbol{U}\cap \bV^c\vert}_{\boldsymbol{U},\bV}\left(\frac{\vert\boldsymbol{U}\cap\bV^c\vert^2}{b^2}(\bar{\mathcal{Z}}_{\bV}-\bar{\mathcal{Z}}_{\boldsymbol{U}\cap \bV^c})(\bar{\mathcal{Z}}_{\bV}-\bar{\mathcal{Z}}_{\boldsymbol{U}\cap \bV^c})^\top\right)
    \\
    &=\mathbb{P}(\vert \boldsymbol{U}\cap \bV^c\vert=1)\mathbb{E}^{\vert \boldsymbol{U}\cap \bV^c\vert=1}_{\boldsymbol{U},\bV}\left(\frac{1}{b^2}(\bar{\mathcal{Z}}_{\bV}-\bar{\mathcal{Z}}_{ \bV^c})(\bar{\mathcal{Z}}_{\bV}-\bar{\mathcal{Z}}_{ \bV^c})^\top\right)
    \\
   & =\mathbb{P}(\vert \boldsymbol{U}\cap \bV^c\vert=1)\mathbb{E}_{\bV}\left(\frac{1}{b^2}(\bar{\mathcal{Z}}_{\bV}-\bar{\mathcal{Z}}_{ \bV^c})(\bar{\mathcal{Z}}_{\bV}-\bar{\mathcal{Z}}_{ \bV^c})^\top\right)
    \\
   & =\frac{1}{Nb}\left(\frac{N}{N-1}\right)^2 \Cov(\bz).
\end{align*}
The proof is completed.
\end{proof}

We provide the following lemma for computing the KL divergence between two joint distributions.
\begin{lemma}
\label{lem:decom_kl_joint}
Let $\bX$, $\bY$, and $\bZ$ be three random variables with $\bX$ and $\bY$ having the same support set. Then the KL divergence between the joint distribution of $(\bX,\bZ)$ and $(\bY,\bZ)$ can be decomposed into
\begin{align*}
    \operatorname{KL}\left((\bX,\bZ)\Vert (\bY,\bZ) \right)=\mathbb{E}_{\bZ} \operatorname{KL}\left((\bX|\bZ)\Vert(\bY|\bZ) \right).
\end{align*}
\end{lemma}
\begin{proof}
By the definition of KL divergence,
\begin{align*}
     \operatorname{KL}\left((\bX,\bZ)\Vert (\bY,\bZ) \right)=&\int \mathbb{P}(\bX,\bZ)\log\frac{\mathbb{P}(\bX,\bZ)}{\mathbb{P}(\bY,\bZ)}
     \\
     =&\int\mathbb{P}(\bZ) \mathbb{P}^{\bZ}(\bX)\log\frac{\mathbb{P}^{\bZ}(\bX)\mathbb{P}(\bZ)}{\mathbb{P}^{\bZ}(\bY)\mathbb{P}(\bZ)}
     \\
     =&\int\mathbb{P}(\bZ)\int \mathbb{P}^{\bZ}(\bX)\log\frac{\mathbb{P}^{\bZ}(\bX)}{\mathbb{P}^{\bZ}(\bY)}
     \\
     =&\mathbb{E}_{\bZ} \operatorname{KL}\left((\bX|\bZ)\Vert(\bY|\bZ) \right).
\end{align*}
The proof is completed.
\end{proof}

In the end of this section, we provide a proof of Lemma \ref{lem: decomposition_kl} using Lemma \ref{lem:decom_kl_joint} for the completeness of this paper.

\begin{proof}[Proof of Lemma \ref{lem: decomposition_kl}]
By Lemma \ref{lem:decom_kl_joint}, we have
\begin{align*}
    &\operatorname{KL}(Q_{0:T}\Vert P_{0:T})
    \\
    &=\operatorname{KL}(Q_{0:T}\Vert P_{0:T})-\operatorname{KL}(Q_{0:T}\Vert (Q_{0:T-1},P_{T|[T-1]}))+\operatorname{KL}(Q_{0:T}\Vert (Q_{0:T-1},P_{T|[T-1]}))
    \\
    &=\int Q_{0:T} \log\frac{Q_{0:T}}{P_{0:T}} -\int Q_{0:T} \log\frac{Q_{0:T}}{Q_{0:T-1}P_{T|[T-1]}} +\mathbb{E}_{Q_{T-1}}\operatorname{KL}\left(Q_{T|[T-1]}\Vert P_{T|[T-1]}\right)
    \\
    & =\int Q_{0:T} \log\frac{Q_{0:T-1}}{P_{0:T-1}} +\mathbb{E}_{Q_{T-1}}\operatorname{KL}\left(Q_{T|[T-1]}\Vert P_{T|[T-1]}\right)
     \\
      &=\int Q_{0:T-1} \log\frac{Q_{0:T-1}}{P_{0:T-1}} +\mathbb{E}_{Q_{T-1}}\operatorname{KL}\left(Q_{T|[T-1]}\Vert P_{T|[T-1]}\right)
      \\
      &=\operatorname{KL}\left( Q_{0:T-1}\left\Vert P_{0:T-1}\right.\right) +\mathbb{E}_{Q_{T-1}}\operatorname{KL}\left(Q_{T|[T-1]}\Vert P_{T|[T-1]}\right).
\end{align*}

The proof is then completed by induction.
\end{proof}
\begin{remark}
In this paper, we focus on the case where $Q_{0:T}$ and $P_{0:T}$ obeys the \emph{Markov Property}, i.e., for any $t$,
\begin{equation*}
    Q_{t|[t-1]}=Q_{t|(t-1)}, \text{ } P_{t|[t-1]}=P_{t|(t-1)}.
\end{equation*}
Therefore, the result in Lemma \ref{lem: decomposition_kl} becomes
\begin{small}
\begin{equation*}
    \operatorname{KL}(Q_{0:T}||P_{0:T})= \sum_{t=1}^{T} \mathbb{E}_{Q_{0:t-1}}\left[\operatorname{KL}\left(Q_{t|(t-1)} \| P_{t|(t-1)}\right)\right].
\end{equation*}
\end{small}
\end{remark}

\subsection{Technical Lemmas in Matrix Analysis}
\label{appen: proof_single_step}
We first provide a sufficient and necessary condition of that two symmetric matrices commute, and the proof can be found from any Linear Algebra Textbook (e.g. \cite{strang1993introduction}).
\begin{lemma}
\label{lem: commute_diag}
Let $\boldsymbol{A}$ and $\boldsymbol{B}$ be two $d\times d$ real symmetric matrices. Then, $\boldsymbol{A}$ and $\boldsymbol{B}$ commute (i.e., $\boldsymbol{A}\boldsymbol{B}=\boldsymbol{B}\boldsymbol{A}$), if and only if there exists an orthogonal matrix $\boldsymbol{O}$ which can diagonalize $\boldsymbol{A}$ and $\boldsymbol{B}$ simultaneously, i.e., both $\boldsymbol{O}^\top \boldsymbol{A}\boldsymbol{O}$ and $\boldsymbol{O}^\top \boldsymbol{B}\boldsymbol{O}$ are diagonal.
\end{lemma}

The next lemma is a key technique to obtain the optimal noise covariance of Theorem \ref{thm: greedy} and Theorem \ref{thm: main_theorem}.

\begin{lemma}
\label{lem: matrix_opt}
 Let $\boldsymbol{B}\in \mathbb{R}^{d\times d}$ be a (fixed) positive definite matrix with eigenvalues $(\beta_1,\cdots,\beta_d)$, where $\beta_i\ge 0$. Let $\boldsymbol{G}\in \mathbb{R}^{d\times d}$ be a positive definite matrix variable with fixed trace $\tr{\boldsymbol{G}}=c$, where $c$ is a positive constant and $c\le \operatorname{tr} (\boldsymbol{B})$. Then the minimum of $\operatorname{tr}(\boldsymbol{G}^{-1}\boldsymbol{B})+\ln \operatorname{det}(\boldsymbol{G})$ is achieved at $\boldsymbol{G}=\boldsymbol{O}^\top \Diag(\alpha_1,\cdots,\alpha_d)\boldsymbol{O}$, where
 \begin{equation*}
    \alpha_i^*= \frac{\sqrt{1-4\lambda^* \beta_i}-1}{-2\lambda^*},
\end{equation*}
$\boldsymbol{O}$ is any orthogonal matrix which diagonalize $\boldsymbol{B}$ as \begin{equation*}
    \boldsymbol{B}=\boldsymbol{O}^\top \Diag (\beta_1,\cdots,\beta_d) \boldsymbol{O},
\end{equation*}
and $\lambda^*\le 0$ is the unique solution of 
\begin{equation}
\label{eq: unique_solution}
    \sum_{i=1}^d\frac{2 \beta_i}{1+\sqrt{1-4\lambda^* \beta_i}}=c.
\end{equation}
\end{lemma}
\begin{remark}
$f(\lambda)=  \sum_{i=1}^d\frac{2 \beta_i}{1+\sqrt{1-4\lambda \beta_i}}$ is a monotonously increasing function with respect to $\lambda$, which guarantees the uniqueness of the solution of $f(\lambda)=c$. 
\end{remark}
Lemma \ref{lem: matrix_opt} is proved via two steps: 1) we first prove for $\bG$  with  fixed eigenvalues, $\tr{\bG^{-1}\boldsymbol{B}}+\ln \operatorname{det} (\bG)$ is optimized if and only if $\bG$ and $\boldsymbol{B}$ share the same eigenvectors; 2) we then calculate the eigenvalues of the optimal $\bG$ using the method of Lagrange multipliers. Theorem \ref{thm: main_theorem} can then be obtained by applying Lemma \ref{lem: matrix_opt} and setting $\bG=\bSig_t(\bS,\bW)$ and $\boldsymbol{B}=\sigma_t\mathbb{I}+\frac{\eta_t^2}{Nb_t}\left(\frac{N}{N-1}\right)^2\bSig^{sd}_{\bS,\bW}$.  We first prove the eigenvectors of $\bG$ agree with those of $\boldsymbol{B}$.
\begin{lemma}
\label{lem: ortho_matrix}
Let $\boldsymbol{G}\in \mathbb{R}^{d\times d}$ be a positive definite matrix variable with fixed eigenvalues $(\alpha_i)_{i=1}^d$. Specifically, let $\alpha_1\ge \alpha_2\ge\cdots\ge\alpha_d>  0$ be all the eigenvalues of $\boldsymbol{G}$, and $\boldsymbol{G}$ can be any element from the following set
\begin{equation*}
    \{\boldsymbol{Q}^\top\Diag(\alpha_1,\cdots,\alpha_d)\boldsymbol{Q}: \boldsymbol{Q} \text{ is orthogonal}\}.
\end{equation*}

Let $\boldsymbol{B}$ be a fixed positive semi-definite matrix, with eigenvalues $(\beta_i)_{i=1}^d$ satisfies $\beta_1\ge \beta_2\cdots\ge\beta_d\ge 0$.
 Then, the optimal (minimal) value of $g (\boldsymbol{G})=\operatorname{tr}(\boldsymbol{G}^{-1}\boldsymbol{B})$ is achieved when 
\begin{equation*}
    \boldsymbol{G}^{*}=\boldsymbol{O}^\top \Diag(\alpha_1,\cdots,\alpha_d) \boldsymbol{O},
\end{equation*}
where $\boldsymbol{O}$ is any orthogonal matrix which diagonalizes $\boldsymbol{B}$  as
\begin{equation*}
    \boldsymbol{B}=\boldsymbol{O}^\top \Diag(\beta_1,\cdots,\beta_d) \boldsymbol{O}
\end{equation*}
and the optimal value of $g(\boldsymbol{G})$ is $\sum_{i=1}^d \frac{\beta_i}{\alpha_i}$.
\end{lemma}
\begin{proof}
 Let $\boldsymbol{G}^*$ be a optimal point of $\operatorname{tr}(\boldsymbol{G}^{-1}\boldsymbol{B})$. We will then obtain the condition of $\boldsymbol{G}^*$ by adding a disturbance. Specifically, let $\boldsymbol{A}$ be an anti-symmetric matrix. Then,
\begin{equation*}
    (\mathbb{I}-\varepsilon \boldsymbol{A}) (\mathbb{I}-\varepsilon \boldsymbol{A})^T= \mathbb{I}+\varepsilon^2 \boldsymbol{A}\boldsymbol{A}^\top.
\end{equation*}
As $\varepsilon$ is small enough, $\mathbb{I}+\varepsilon^2\boldsymbol{A}\boldsymbol{A}^\top$ is inevitable, and positive definite. Therefore, $(\mathbb{I}-\varepsilon \boldsymbol{A})(\mathbb{I}+\varepsilon^2\boldsymbol{A}\boldsymbol{A}^\top)^{-\frac{1}{2}}$ is orthogonal. As $\varepsilon \rightarrow 0$, 
\begin{equation*}
    \lim_{\varepsilon\rightarrow 0}(\mathbb{I}-\varepsilon \boldsymbol{A})(\mathbb{I}+\varepsilon^2\boldsymbol{A}\boldsymbol{A}^\top)^{-\frac{1}{2}}=\mathbb{I},
\end{equation*}
and 
\begin{equation*}
   (\mathbb{I}-\varepsilon \boldsymbol{A})(\mathbb{I}+\varepsilon^2\boldsymbol{A}\boldsymbol{A}^\top)^{-\frac{1}{2}} -\mathbb{I}=-\varepsilon\boldsymbol{A}+\boldsymbol{o}(\varepsilon).
\end{equation*}
Since $\boldsymbol{G}^*$ is an optimal point of $\operatorname{tr}(\boldsymbol{G}^{-1}\boldsymbol{B})$, we have 
\begin{align*}
    \operatorname{tr}((\boldsymbol{G}^*)^{-1}\boldsymbol{B})
    \le &\operatorname{tr} \left((\mathbb{I}-\varepsilon \boldsymbol{A})(\mathbb{I}+\varepsilon^2\boldsymbol{A}\boldsymbol{A}^\top)^{-\frac{1}{2}}(\boldsymbol{G}^*)^{-1}\left((\mathbb{I}-\varepsilon \boldsymbol{A})(\mathbb{I}+\varepsilon^2\boldsymbol{A}\boldsymbol{A}^\top)^{-\frac{1}{2}}\right)^\top\boldsymbol{B}\right)\\
    =& \operatorname{tr} \left( (\mathbb{I}-\varepsilon \boldsymbol{A})(\mathbb{I}+\varepsilon^2\boldsymbol{A}\boldsymbol{A}^\top)^{-\frac{1}{2}}(\boldsymbol{G}^*)^{-1}(\mathbb{I}+\varepsilon^2\boldsymbol{A}\boldsymbol{A}^\top)^{-\frac{1}{2}}(\mathbb{I}+\varepsilon \boldsymbol{A})\boldsymbol{B}\right),
\end{align*}
which further leads to 
\begin{align*}
    -\varepsilon\operatorname{tr} \left( \boldsymbol{A}(\boldsymbol{G}^*)^{-1}\boldsymbol{B}\right)
    +\varepsilon\operatorname{tr} \left( (\boldsymbol{G}^*)^{-1}\boldsymbol{A}\boldsymbol{B}\right)+\boldsymbol{o}(\varepsilon) 
    \ge 0.
\end{align*}
By letting $\varepsilon\rightarrow 0$, we further have
\begin{align*}
     -\operatorname{tr} \left( \boldsymbol{A}(\boldsymbol{G}^*)^{-1}\boldsymbol{B}\right)
    +\operatorname{tr} \left( (\boldsymbol{G}^*)^{-1}\boldsymbol{A}\boldsymbol{B}\right)=0,
\end{align*}
which further leads to 
\begin{align}
\nonumber
    0=&-\operatorname{tr} \left( \boldsymbol{A}(\boldsymbol{G}^*)^{-1}\boldsymbol{B}\right)
    +\operatorname{tr} \left( (\boldsymbol{G}^*)^{-1}\boldsymbol{A}\boldsymbol{B}\right)
    \\
\nonumber
    =&\operatorname{tr} \left( \boldsymbol{A}^\top(\boldsymbol{G}^*)^{-1}\boldsymbol{B}\right)
    +\operatorname{tr} \left(\boldsymbol{B} (\boldsymbol{G}^*)^{-1}\boldsymbol{A}\right)
    \\
\label{eq:adding_disturb}
    =&2\operatorname{tr} \left(\boldsymbol{B} (\boldsymbol{G}^*)^{-1}\boldsymbol{A}\right).
\end{align}

Since Eq.(\ref{eq:adding_disturb}) holds for any anti-symmetry matrix $\boldsymbol{A}$, let $\boldsymbol{A}=\boldsymbol{E}_{i,j}-\boldsymbol{E}_{j,i}$, where $i,j\in [d]$ and $i\ne j$. By Eq.(\ref{eq:adding_disturb}), we have 
\begin{align*}
    \left(\boldsymbol{B} (\boldsymbol{G}^*)^{-1}\right)_{i,j}
    = \left(\boldsymbol{B} (\boldsymbol{G}^*)^{-1}\right)_{j,i},
\end{align*}
which further leads to 
\begin{align*}
    \boldsymbol{B} (\boldsymbol{G}^*)^{-1}
    = \left(\boldsymbol{B} (\boldsymbol{G}^*)^{-1}\right)^\top= (\boldsymbol{G}^*)^{-\top}\boldsymbol{B}^\top= (\boldsymbol{G}^*)^{-1}\boldsymbol{B}.
\end{align*}
By simple rearranging, we have
\begin{align*}
    \boldsymbol{G}^{*}\boldsymbol{B}=\boldsymbol{B} \boldsymbol{G}^{*}.
\end{align*}
Therefore, by Lemma \ref{lem: commute_diag}, we have that there exists an orthogonal matrix $\boldsymbol{O}_0$, such that both $\boldsymbol{O}_0\boldsymbol{G}^{*} \boldsymbol{O}_0^\top$ and $\boldsymbol{O}_0\boldsymbol{B} \boldsymbol{O}_0^\top$ are diagonal. By multiplying a permutation matrix, we further have there exists an orthogonal matrix $\tilde{\boldsymbol{O}}$ such that $\tilde{\boldsymbol{O}}\boldsymbol{G}^*\tilde{\boldsymbol{O}}^\top$ is diagonal, and \begin{equation}
\label{eq: diag_B}
    \tilde{\boldsymbol{O}}\boldsymbol{B}\tilde{\boldsymbol{O}}^\top=\Diag(\beta_1,\cdots,\beta_d).
\end{equation}

Since $\tilde{\boldsymbol{O}}\boldsymbol{G}^*\tilde{\boldsymbol{O}}^\top$ is diagonal, there exists a permutation mapping $\mathcal{T}: [d]\rightarrow[d]$, such that
\begin{equation}
\label{eq: diag_G}
    \tilde{\boldsymbol{O}}\boldsymbol{G}^*\tilde{\boldsymbol{O}}^\top=\Diag \left(\alpha_{\mathcal{T}(1)},\cdots, \alpha_{\mathcal{T}(d)}\right).
\end{equation}
Denote the order of $\beta_i$ $(i=1,2,\cdots,d)$ as  
\begin{align}
\label{eq: reordering_beta} \beta_1=\cdots=\beta_{s_1}>\beta_{s_1+1}=\cdots=\beta_{s_1+s_2}>\cdots
   >\beta_{\sum_{i=1}^{k-1} s_i+1}=\cdots=\beta_{\sum_{i=1}^{k} s_i}>0,
\end{align}
where $\sum_{i=1}^{k} s_i=d$, and we denote $s_0=0$. Since $\boldsymbol{G}^*$ is the optimal point of $\operatorname{tr}((\boldsymbol{G}^*)^{-1}\boldsymbol{B})$, for any $1\le i<j\le d$ and $\beta_i>\beta_j$, we have $\alpha_{\mathcal{T}(i)}>\alpha_{\mathcal{T}(j)}$: otherwise, let 
\begin{equation*}
    \boldsymbol{G}'= \tilde{\boldsymbol{O}}^\top\Diag\left(\alpha_{\mathcal{T}(1)},\cdots,\alpha_{\mathcal{T}(i-1)},\alpha_{\mathcal{T}(j)},\alpha_{\mathcal{T}(i+1)},\cdots,\alpha_{\mathcal{T}(j-1)},\alpha_{\mathcal{T}(i)},\alpha_{\mathcal{T}(j+1)},\cdots, \alpha_{\mathcal{T}(d)}\right)\tilde{\boldsymbol{O}},
\end{equation*}
we have 
\begin{equation*}
    \operatorname{tr}((\boldsymbol{G}^*)^{-1}\boldsymbol{B})>\operatorname{tr}((\boldsymbol{G}')^{-1}\boldsymbol{B}),
\end{equation*}
which contradicts  that $\boldsymbol{G}^*$ is optimal.

Therefore, $\mathcal{T}(\sum_{i=1}^j s_i+1),\cdots,\mathcal{T}(\sum_{i=1}^{j+1} s_i)$ is then a permutation of $\sum_{i=1}^j s_i+1,\cdots,\sum_{i=1}^{j+1} s_i$, and there exists permutation matrix $\boldsymbol{Q}$ such that 
\begin{equation}
\label{eq: construction_Q}
\boldsymbol{Q}=\Diag\left(\boldsymbol{Q}_1,\cdots,\boldsymbol{Q}_k\right),
\end{equation}
where $\boldsymbol{Q}_i$ is a $s_i\times s_i$ permutation sub-matrix, such that,
\begin{equation}
\label{eq: reorder_G}
    \boldsymbol{Q}\Diag \left(\alpha_{\mathcal{T}(1)},\cdots, \alpha_{\mathcal{T}(d)}\right)\boldsymbol{Q}^\top=\Diag \left(\alpha_{1},\cdots, \alpha_{d}\right).
\end{equation}

Furthermore, by Eq.(\ref{eq: construction_Q}) and Eq.(\ref{eq: reordering_beta}), we have
\begin{equation}
\label{eq: reorder_B}
    \boldsymbol{Q}\Diag \left(\beta_{1},\cdots, \beta_{d}\right) \boldsymbol{Q}^\top=\Diag \left(\beta_{1},\cdots, \beta_{d}\right).
\end{equation}

Therefore, by Eqs.(\ref{eq: diag_B}), (\ref{eq: diag_G}), (\ref{eq: reorder_G}), and (\ref{eq: reorder_B}), we have 
\begin{gather*}
    \boldsymbol{Q}\tilde{\boldsymbol{O}} \boldsymbol{B}\left( \boldsymbol{Q}\tilde{\boldsymbol{O}} \right)^\top=\Diag\left(\beta_{1},\cdots, \beta_{d}\right),
    \\
    \boldsymbol{Q}\tilde{\boldsymbol{O}} \boldsymbol{G}^*\left( \boldsymbol{Q}\tilde{\boldsymbol{O}} \right)^\top=\Diag\left(\alpha_{1},\cdots, \alpha_{d}\right).
\end{gather*}

Furthermore, 
\begin{align*}
    &\operatorname{tr} \left(\boldsymbol{G}^{-1}\boldsymbol{B}\right)
    \\
   & =\operatorname{tr}\left(\left(\boldsymbol{Q}\tilde{\boldsymbol{O}}\right)^\top \Diag\left(\alpha^{-1}_{1},\cdots, \alpha^{-1}_{d}\right)\left(\boldsymbol{Q}\tilde{\boldsymbol{O}}\right)\left(\boldsymbol{Q}\tilde{\boldsymbol{O}}\right)^\top \Diag\left(\beta_{1},\cdots, \beta_{d}\right)\left(\boldsymbol{Q}\tilde{\boldsymbol{O}}\right)\right)
    \\
    &=\sum_{i=1}^d \frac{\beta_i}{\alpha_i}.
\end{align*}
Therefore, the optimal value of $\tr(\boldsymbol{G}^{-1}\boldsymbol{B})$ is $\sum_{i=1}^d \frac{\beta_i}{\alpha_i}$, and the corresponding optimal point $\boldsymbol{G}^*$ belongs to the following set 
\begin{equation*}
    \mathcal{G}=\{\boldsymbol{O}^\top \Diag(\alpha_1,\cdots,\alpha_d)\boldsymbol{O}: \boldsymbol{B}=\boldsymbol{O}^\top \Diag(\beta_1,\cdots,\beta_d)\boldsymbol{O}\}.
\end{equation*}

On the other hand, it is easy to verify that for any element $\boldsymbol{G}\in \mathcal{G}$, 
\begin{equation*}
    \tr(\boldsymbol{G}^{-1}\boldsymbol{B})=\sum_{i=1}^d \frac{\beta_i}{\alpha_i}.
\end{equation*}
The proof is completed.
\end{proof}

Lemma \ref{lem: ortho_matrix} indicates that  with eigenvalues fixed, the eigenvectors of $\boldsymbol{G}$ should agree with those of $\boldsymbol{B}$ by the order of eigenvalues. We then provide the following lemma to determine the optimal eigenvalues.

\begin{lemma}
\label{lem: eigen_value_matrix}
Let $\beta_1, \beta_2,\cdots,  \beta_d$ be a series of fixed positive reals. Let $\alpha_1, \alpha_2, \cdots,\alpha_d\in \mathbb{R}^{+}$ be a series of real variables with constraint $\sum_{i=1}^d \alpha_i=c$, where $c$ is a positive real constant which satisfies $c\le \sum_{i=1}^d \beta_i$. Then the minimum of function 
\begin{equation*}
    f(\alpha_1,\cdots,\alpha_d) =\sum_{i=1}^d \frac{\beta_i}{\alpha_i}+\sum_{i=1}^d\ln \alpha_i
\end{equation*}
is achieved at 
\begin{equation*}
    \alpha_i^*= \frac{\sqrt{1-4\lambda^* \beta_i}-1}{-2\lambda^*},
\end{equation*}
where $\lambda^*\le 0$ is the unique solution of 
\begin{equation*}
    \sum_{i=1}^d\frac{2 \beta_i}{1+\sqrt{1-4\lambda^* \beta_i}}=c.
\end{equation*}
\end{lemma}
\begin{proof}
We find the minimum of $f$ under the constraint that $\alpha_1+\cdots+\alpha_d=c$ by the method of Lagrange Multiplier. Specifically, as for any $i\in [d]$, $\alpha_i\rightarrow 0^+$ or  $\alpha_i\rightarrow c^-$ will lead to $f(\alpha_1,\cdots,\alpha_d)\rightarrow\infty$, we have that for any global optimal (minimal) point  $(\alpha_1^*,\cdots,\alpha_d^*)$ of $f$ under the constraint $\alpha_1+\cdots+\alpha_d=c$, we have that there exist a real $\lambda^*$, such that $((\alpha_1^*,\cdots,\alpha_d^*),\lambda^*)$ is a saddle point of $\mathcal{L}((\alpha_1,\cdots,\alpha_d),\lambda)$, which is defined as
\begin{equation*}
    \mathcal{L}((\alpha_1,\cdots,\alpha_d),\lambda)=f(\alpha_1,\cdots,\alpha_d)+\lambda(c-\alpha_1-\cdots-\alpha_d).
\end{equation*}

By taking partial derivative of $\mathcal{L}$ with respect to $\alpha_i$, we have 
\begin{equation}
\label{eq: quadratic}
    -\lambda^*=-\frac{1}{\alpha_i}+\frac{\beta_i}{\alpha_i^2}=\frac{\beta_i-\alpha_i}{\alpha_i^2},
\end{equation}
which further leads to 
\begin{equation*}
     \sum_{i=1}^d\beta_i-c=\sum_{i=1}^d\left(\beta_i-\alpha_i\right)=-\lambda^*\left(\sum_{i=1}^d\alpha_i^2\right).
\end{equation*}

Since $\sum_{i=1}^d \beta_i\ge  c$, we have $\lambda^*\le 0$. Therefore, for any $i\in [d]$, the quadratic equation $\beta_i x^2-x+\lambda^*=0$ has only one positive solution $\frac{1+\sqrt{1-4\lambda^* \beta_i}}{2 \beta_i}$, and 
\begin{equation*}
    \alpha_i^*= \frac{2 \beta_i}{1+\sqrt{1-4\lambda^* \beta_i}}=\frac{\sqrt{1-4\lambda^* \beta_i}-1}{-2\lambda^*}.
\end{equation*}

On the other hand, by taking derivative of $\mathcal{L}$ with respect to $\lambda^*$, we have
\begin{equation}
\label{eq: solution_lambda}
    \sum_{i=1}^d \alpha_i^*=\sum_{i=1}^d\frac{2 \beta_i}{1+\sqrt{1-4\lambda^* \beta_i}}=c.
\end{equation}
Since $\sum_{i=1}^d \alpha_i^*=\sum_{i=1}^d\frac{2 \beta_i}{1+\sqrt{1-4\lambda^* \beta_i}}$ is a monotonously increasing function of $\lambda^*$, there is only one solution of $\lambda^*$ of Eq.(\ref{eq: solution_lambda}).

The proof is completed.
\end{proof}

The proof of Lemma \ref{lem: matrix_opt} can then be obtained by combining Lemma \ref{lem: ortho_matrix} and Lemma \ref{lem: eigen_value_matrix} together.

\begin{proof}[Proof of Lemma \ref{lem: matrix_opt}]
The original optimization problem can be written as 
\begin{equation*}
    \min_{\tr(\boldsymbol{G})=c}  \tr\left(\boldsymbol{G}^{-1}\boldsymbol{B}\right)+\ln\left(\det \boldsymbol{G}\right),
\end{equation*}
which can be further decomposed into 
\begin{align*}
&\min_{\tr(\boldsymbol{G})=c}  \tr\left(\boldsymbol{G}^{-1}\boldsymbol{B}\right)+\ln\left(\det \boldsymbol{G}\right)
\\
&=\min_{\substack{\sum_{i=1}^d \alpha_i=c\\ \alpha_1\ge\cdots\ge \alpha_d>0}} \min_{\boldsymbol{O} \in \mathcal{O}(d)} \left( \tr\left(\boldsymbol{O}^\top \Diag\left(\alpha_1^{-1},\cdots,\alpha_d^{-1}\right) \boldsymbol{O}\boldsymbol{B}\right)+\sum_{i=1}^d \ln \alpha_i\right)
\\
&\overset{(*)}{=} \min_{\substack{\sum_{i=1}^d \alpha_i=c}} \left(\sum_{i=1}^d \frac{\beta_i}{\alpha_i}+\sum_{i=1}^d \ln \alpha_i\right) 
\\
&\overset{(**)}{=}\sum_{i=1}^d \frac{1+\sqrt{1-4\lambda^* \beta_i}}{2}+\sum_{i=1}^d \ln\frac{2 \beta_i}{1+\sqrt{1-4\lambda^* \beta_i}}, 
\end{align*}
where Eq. $(*)$ is due to Lemma \ref{lem: ortho_matrix}, Eq. $(**)$ is due to Lemma \ref{lem: eigen_value_matrix}, and $\lambda^*\le 0$ is the unique solution of 
\begin{equation*}
    \sum_{i=1}^d\frac{2 \beta_i}{1+\sqrt{1-4\lambda^* \beta_i}}=c.
\end{equation*}

Furthermore, the optimal point of $\tr(\boldsymbol{G}^{-1}\boldsymbol{B})+\ln (\operatorname{det} \boldsymbol{G})$ can be calculated as 
\begin{align*}
    &\arg\min_{\tr(\boldsymbol{G})=c}  \tr\left(\boldsymbol{G}^{-1}\boldsymbol{B}\right)+\ln\left(\det \boldsymbol{G}\right)
    \\
    &=\left\{\boldsymbol{O}^\top \Diag\left(\frac{\sqrt{1-4\lambda^* \beta_1}-1}{-2\lambda^*},\cdots,\frac{\sqrt{1-4\lambda^* \beta_d}-1}{-2\lambda^*}\right)\boldsymbol{O}: \boldsymbol{B}=\boldsymbol{O}^\top \Diag(\beta_1,\cdots,\beta_d)\boldsymbol{O}
    ,\right.
    \\
    &\;\;\left.\lambda^*= \arg_{\lambda}\left( \sum_{i=1}^d\frac{2 \beta_i}{1+\sqrt{1-4\lambda \beta_i}}=c\right)\right\}.
\end{align*}

The proof is completed.
\end{proof}

\section{Supplementary Materials of Section \ref{sec: obtaining_constraint}}
\label{appendix: proof of lemma optimization}
\begin{proof}[Proof of Lemma \ref{lem:state-dependent sgld optimization}]
The $\beta$-smooth condition gives
\begin{equation}
\label{eq:convergence_eq0}
\mathcal{R}_{\bS}(\bW_{t+1})\leq \mathcal{R}_{\bS}(\bW_{t})+\langle\nabla \mathcal{R}_{\bS}(\bW_{t}),\bW_{t+1}-\bW_{t}\rangle+\frac{\beta}{2}\Vert \bW_{t+1}-\bW_{t}\Vert^{2}.
\end{equation}

Based on the update rule Eq.(\ref{eq:state-dependent SGLD}),  we have
\begin{equation}
\begin{aligned}
\label{eq:convergence_eq1}
\bW_{t+1}-\bW_{t}=-\eta_{t+1}\nabla \mathcal{R}_{\bS_{\bV_{t+1}}}(\bW_{t})+\varepsilon_{t+1},
\end{aligned}
\end{equation}
where $\varepsilon_{t+1}\sim\mathcal{N}(0,\bSig_{t+1}(\bS,\bW_t))$.

Take expectation on Eq.(\ref{eq:convergence_eq0}) with respect to $\bW_{t+1}|\bW_{t}$, by $\mathbb{E}^{\bW_{t}}(\nabla \mathcal{R}_{\bS_{\bV_{t+1}}}(\bW_t)) = \nabla \mathcal{R}_{\bS}(\bW_t)$, 
\begin{equation}
\begin{aligned}
\label{eq:convergence_eq2}
\mathbb{E}^{\bW_{t}}[\mathcal{R}_{\bS}(\bW_{t+1})]\leq \mathcal{R}_{\bS}(\bW_{t})-\eta_{t+1}\Vert \nabla \mathcal{R}_{\bS}(\bW_t)\Vert^2+\frac{\beta}{2} \mathbb{E}^{\bW_t}\Vert \bW_{t+1}-\bW_{t}\Vert^2.
\end{aligned}
\end{equation}

Furthermore,
\begin{align}
\nonumber
    &\mathbb{E}^{\bW_t}\Vert \bW_{t+1}-\bW_{t}\Vert^2
    \\
\nonumber
    &=\mathbb{E}^{\bW_t}\Vert -\eta_{t+1} \nabla \mathcal{R}_{\bS_{\bV_{t+1}}}(\bW_t) +\varepsilon_{t+1}\Vert^2
    \\
\nonumber
    &\overset{(*)}{=}\mathbb{E}^{\bW_t}\Vert -\eta_{t+1} \nabla \mathcal{R}_{\bS_{\bV_t}}(\bW_t)\Vert +\mathbb{E}^{\bW_t}\Vert \varepsilon_{t+1}\Vert^2
    \\
\nonumber
    &=\eta_{t+1}^2\mathbb{E}^{\bW_t}\Vert  \nabla \mathcal{R}_{\bS_{\bV_{t+1}}}(\bW_t)-\nabla \mathcal{R}_{\bS}(\bW_t)+\nabla \mathcal{R}_{\bS}(\bW_t)\Vert^2 +\bSig_{t+1}(\bS,\bW_{t})
    \\
\nonumber
    &=\eta_{t+1}^2\mathbb{E}^{\bW_t}\Vert  \nabla \mathcal{R}_{\bS_{\bV_{t+1}}}(\bW_t)-\nabla \mathcal{R}_{\bS}(\bW_t)\Vert^2+\eta_{t+1}^2\Vert\nabla \mathcal{R}_{\bS}(\bW_t)\Vert^2 +\bSig_{t+1}(\bS,\bW_{t})
    \\
\label{eq: expected_squared}
    &=\frac{\eta_{t+1}^2}{N-1}\frac{N-b_{t+1}}{b_{t+1}} \bSig^{sd}_{\bS,\bW_t}+\eta_{t+1}^2\Vert\nabla \mathcal{R}_{\bS}(\bW_t)\Vert^2 +\bSig_{t+1}(\bS,\bW_{t}),
\end{align}
where Eq.$(*)$ is due to $\varepsilon_{t+1}$ is independent of $\bV_{t+1}$, and $\mathbb{E}^{\bW_t} \varepsilon_{t+1}=0$.

Applying Eq.(\ref{eq: expected_squared}) back to Eq.(\ref{eq:convergence_eq2}) completes the proof.

\end{proof}

\section{Supplementary Materials of Section \ref{sec: general_information_theoretical_generalization_bound}}
\label{appen:4}
\subsection{Example to illustrate the difficulty to apply Proposition \ref{prop: negrea_2019} to solve \textbf{Problem 1}}
\label{subsec: appen_example}
In this section, we show an example to demonstrate the difficulty for tackling \textbf{Problem 1} through Proposition \ref{prop: negrea_2019}. To start with, by the definition of state-dependent SGLD (Eq.(\ref{eq:state-dependent SGLD})), covariance $\bSig_{[T]}$ is independent of $\bJ$ and $\bV_{[T]}$. Therefore, the square root separates the expectation with respect to $\bV_{[T]}$  and $\bJ$ from the KL divergence term in the generalization bound
\begin{equation*}
    \mathbb{E}_{\bS,\bV_{[T]},\bJ}\sqrt{\frac{(a_2-a_1)^2}{2}\sum_{s=1}^T\mathbb{E}_{Q_{s-1}^{\bS,\bV_{[T]}}}\operatorname{KL}\left(Q_{s|(s-1)}^{\bS,\bV_{[T]}}\left\Vert P_{s|(s-1)}^{\bJ,\bS_{\bJ},\bV_{[T]}}\right.\right)},
\end{equation*}
which makes the dependency of the bound on $\bSig_{[T]}$ even more complex. However, even though we change the optimization target into
\begin{equation}
\label{eq: optimization_target_example}
    \mathbb{E}_{\bS}\sqrt{\frac{(a_2-a_1)^2}{2}\mathbb{E}_{\bV_{[T]},\bJ}\sum_{s=1}^T\mathbb{E}_{Q_{s-1}^{\bS,\bV_{[T]}}}\operatorname{KL}\left(Q_{s|(s-1)}^{\bS,\bV_{[T]}}\left\Vert P_{s|(s-1)}^{\bJ,\bS_{\bJ},\bV_{[T]}}\right.\right)},
\end{equation}
which is still a generalization bound by Jensen's Inequality, we demonstrate that the dependency on $\bSig_{[T]}$ is still too complex to tackle as follows.

To optimize Eq.(\ref{eq: optimization_target_example}) with respect to $\bSig_{[T]}(\bS,\cdot)$ for fixed $\bS$, we are actually seeking the optimal point of the following optimization problem:
\begin{small}
\begin{equation}
 \label{eq:optimization_prop_1}
    \bSig^*_{[T]}(\bS,\cdot)=\arg\min_{\bSig_{[T]}(\bS,\cdot)}\sqrt{\mathbb{E}_{\bV_{[T]},\bJ}\sum_{s=1}^T\mathbb{E}_{Q_{s-1}^{\bS,\bV_{[T]}}}\operatorname{KL}\left(Q_{s|(s-1)}^{\bS,\bV_{[T]}}\left\Vert P_{s|(s-1)}^{\bJ,\bS_{\bJ},\bV_{[T]}}\right.\right)}.
\end{equation}
\end{small}

 However, we will show it is technically hard to solve Eq. (\ref{eq:optimization_prop_1}). As discussed in Section \ref{sec: difficulty_traditional}, for any fixed index $i\in[T]$, Eq.(\ref{eq:optimization_prop_1}) depends on $\bSig_s(\bS,\cdot)$ through both
$\mathbb{E}_{\bV_{[T]},\bJ}\mathbb{E}_{Q_{s-1}^{\bS,\bV_{[T]}}}$ $\operatorname{KL}(Q_{s|(s-1)}^{\bS,\bV_{[T]}}\Vert P_{s|(s-1)}^{\bJ,\bS_{\bJ},\bV_{[T]}}) $ and $\mathbb{E}_{\bV_{[T]},\bJ}\mathbb{E}_{Q_{i-1}^{\bS,\bV_{[T]}}}\operatorname{KL}(Q_{i|(i-1)}^{\bS,\bV_{[T]}}\Vert P_{i|(i-1)}^{\bJ,\bS_{\bJ},\bV_{[T]}}) $ for  $\forall i>s$. Specifically, we adopt the update rule for prior  for all the steps and posterior  for all steps $t\ne s$ to be the isotropic SGLD in \cite{negrea2019information}, i.e., 
\begin{small}
\begin{gather*}
    \text{Posterior: }\bW_t=\bW_{t-1}-\eta_t \nabla \mathcal{R}_{\bS_{\bV_t}} (\bW_{t-1})+\mathcal{N}(\boldsymbol{0},\sigma_t\mathbb{I})
    \\
    \text{Prior: }\bW_t=\bW_{t-1}-\eta_t \left(\frac{\vert\bV_t \cap \bJ \vert}{\vert\bV_t \vert}\nabla \mathcal{R}_{\bS_{\bV_t \cap \bJ} }\left(\bW_{t-1}\right)+\frac{\vert\bV_t \cap \bJ^c \vert}{\vert\bV_t\vert }\nabla \mathcal{R}_{\bS_{ \bJ} }\left(\bW_{t-1}\right)\right)    + \mathcal{N}\left(\boldsymbol{0},\sigma_t \mathbb{I}\right),
\end{gather*}
\end{small}
while we only optimize the noise covariance $\bSig_s(\bS,\cdot)$ of step $s$:
\begin{equation*}
    \bW_s=\bW_{s-1}-\eta_s \nabla \mathcal{R}_{\bS_{\bV_s}} (\bW_{s-1})+\mathcal{N}(\boldsymbol{0},\bSig_s(\bS,\bW_{s-1})).
\end{equation*}
By simple calculation, for any step $t\in[T]$, given the same $\bW_{t-1}$, $\bV_t$, $\bJ$, and $\bS$, the mean between the prior and posterior can be calculated as 
\begin{align}
\nonumber
    &\mu^{\bS,\bV_t,\bJ,\bW_{t-1}}
    \\\nonumber
   & =-\eta_t \left(\frac{\vert\bV_t \cap \bJ \vert}{\vert\bV_t \vert}\nabla \mathcal{R}_{\bS_{\bV_t \cap \bJ} }\left(\bW_{t-1}\right)+\frac{\vert\bV_t \cap \bJ^c \vert}{\vert\bV_t\vert }\nabla \mathcal{R}_{\bS_{ \bJ} }\left(\bW_{t-1}\right)\right)+\eta_t \nabla \mathcal{R}_{\bS_{\bV_t}} (\bW_{t-1})
    \\
    \label{eq:difference_mean}
    &=\eta_t \frac{\vert\bV_t \cap \bJ^c \vert}{\vert\bV_t\vert } \left(\nabla \mathcal{R}_{\bS_{\bV_t\cap \bJ^c}} (\bW_{t-1})-\nabla \mathcal{R}_{\bS_{\bJ}} (\bW_{t-1})\right).
\end{align}
Therefore, by Lemma \ref{lem: kl_gau} and Lemma \ref{lem: sampling}, the expected KL divergence $\mathbb{E}_{\bV_{[T]},\bJ}\mathbb{E}_{Q_{i-1}^{\bS,\bV_{[T]}}}\operatorname{KL}$ $(Q_{i|(i-1)}^{\bS,\bV_{[T]}}\Vert P_{i|(i-1)}^{\bJ,\bS_{\bJ},\bV_{[T]}}) $ can be calculated as
\begin{align*}
    &\mathbb{E}_{\bV_{[T]},\bJ}\mathbb{E}_{Q_{i-1}^{\bS,\bV_{[T]}}}\operatorname{KL}\left(Q_{i|(i-1)}^{\bS,\bV_{[T]}}\left\Vert P_{i|(i-1)}^{\bJ,\bS_{\bJ},\bV_{[T]}}\right.\right)
    \\
    &=\frac{1}{2}\mathbb{E}_{\bV_{[T]},\bJ}\mathbb{E}_{Q_{i-1}^{\bS,\bV_{[T]}}} \left(\sigma_i^{-1}\mu^{\bS,\bV_t,\bJ,\bW_{i-1}}\left(\mu^{\bS,\bV_t,\bJ,\bW_{i-1}}\right)^\top\right)
    \\
    &=\frac{1}{2\sigma_i} \frac{1}{Nb_i}\left(\frac{N}{N-1}\right)^2 \mathbb{E}_{\bV_{[i-1]}}\mathbb{E}_{Q_{i-1}^{\bS,\bV_{[i-1]}}}\bSig^{sd}_{\bS,\bW_{i-1}}.
\end{align*}
Therefore, the exact form of $\mathbb{E}_{\bV_{[T]},\bJ}\mathbb{E}_{Q_{i-1}^{\bS,\bV_{[T]}}}\operatorname{KL}(Q_{i|(i-1)}^{\bS,\bV_{[T]}}\Vert P_{i|(i-1)}^{\bJ,\bS_{\bJ},\bV_{[T]}})$ requires taking expectation to $\bSig^{sd}_{\bS,\bW_{i-1}}$ with respect to Gaussian distribution with covariance $\bSig_s$, and can be complex due to the complex structure of the model. Specifically, if $i=s+1$, then $\mathbb{E}_{\bV_{[T]},\bJ}\mathbb{E}_{Q_{i-1}^{\bS,\bV_{[T]}}}$ $\operatorname{KL}(Q_{i|(i-1)}^{\bS,\bV_{[T]}}\Vert P_{i|(i-1)}^{\bJ,\bS_{\bJ},\bV_{[T]}})$ can be further written as 
\begin{align*}
    &\mathbb{E}_{\bV_{[T]},\bJ}\mathbb{E}_{Q_{s}^{\bS,\bV_{[T]}}}\operatorname{KL}\left(Q_{s+1|s}^{\bS,\bV_{[T]}}\left\Vert P_{s+1|s}^{\bJ,\bS_{\bJ},\bV_{[T]}}\right.\right)
    \\
    =&\frac{1}{2} \frac{1}{Nb_{s+1}}\left(\frac{N}{N-1}\right)^2 \mathbb{E}_{\bV_{[s]}}\mathbb{E}_{Q_{s-1}^{\bS,\bV_{[s-1]}}}\mathbb{E}_{Q_{s|(s-1)}^{\bS,\bV_s}}\bSig^{sd}_{\bS,\bW_{s}}.
\end{align*}
Therefore, we need to optimize  $\mathbb{E}_{\bV_s}\mathbb{E}_{Q_{s|(s-1)}^{\bS,\bV_s}}\bSig^{sd}_{\bS,\bW_{s}}$, which can be further written as
\begin{align*}
    \mathbb{E}_{\bV_s}\mathbb{E}_{Q_{s|(s-1)}^{\bS,\bV_s}}\bSig^{sd}_{\bS,\bW_{s}}= \mathbb{E}_{\bV_s}\mathbb{E}_{\mathcal{N}(-\eta_s\nabla \mathcal{R}_{\bS_{\bV_s}}(\bW_{s-1}),\bSig_s(\bS,\bW_{s-1}))}\bSig^{sd}_{\bS,\bW_{s}}.
\end{align*}
The explicit form of $\mathbb{E}_{\bV_s}\mathbb{E}_{\mathcal{N}(-\eta_s\nabla \mathcal{R}_{\bS_{\bV_s}}(\bW_{s-1}),\bSig_s(\bS,\bW_{s-1}))}\bSig^{sd}_{\bS,\bW_{s}}$ can be obtained only when $\bSig^{sd}_{\bS,\bW_s}$ is some simple functions with respect to $\bW_s$ (e.g. quadratic functions), which makes the optimal of $\mathbb{E}_{\bV_s}\mathbb{E}_{\mathcal{N}(-\eta_s\nabla \mathcal{R}_{\bS_{\bV_s}}(\bW_{s-1}),\bSig_s(\bS,\bW_{s-1}))}\bSig^{sd}_{\bS,\bW_{s}}$ complicated due to the complex structure of $\mathcal{R}_{\bS}$ and $\bSig^{sd}_{\bS,\bW}$ in pratical learning problems.

\subsection{Proof of Theorem \ref{thm: reversed_generalization}}
\label{appen: proof_generalization_2}
\begin{proof}[Proof of Theorem \ref{thm: reversed_generalization}]
For any two random measures $P^{\bJ,\bS_{\bJ},\bV_{[T]}},Q^{\bS,\bV_{[T]}}$, by the Donsker-Varadhan variational formula \cite{boucheron2013concentration}, for any function $g$ satisfying $Q^{\bS,\bV_{[T]}}(\exp g)<\infty$, we have
\begin{align*}
	\operatorname{KL}(P^{\bJ,\bS_{\bJ},\bV_{[T]}}||Q^{\bS,\bV_{[T]}}) 
	\geq P^{\bJ,\bS_{\bJ},\bV_{[T]}}(g) - Q^{\bS,\bV_{[T]}}(g) - \log Q^{\bS,\bV_{[T]}}\left( \exp(g-Q^{\bS,\bV_{[T]}}(g))\right).
\end{align*}

Letting $g(\bW) = \lambda\left(\hat{\mathcal{R}}_{\bS_{\bJ^c}}(\bW)-\mathcal{R}_{\mathcal{D}}(\bW)  \right)$, we further have
\begin{align*}
&\operatorname{KL}(P^{\bJ,\bS_{\bJ},\bV_{[T]}}||Q^{\bS,\bV_{[T]}})
\\
 & \geq  \lambda\left(\mathcal{R}_{\mathcal{D}}(Q^{\bS,\bV_{[T]}}) - \hat{\mathcal{R}}_{\bS_{\bJ^c}}(Q^{\bS,\bV_{[T]}}) - \left(\mathcal{R}_{\mathcal{D}}(P^{\bJ,\bS_{\bJ},\bV_{[T]}}) - \hat{\mathcal{R}}_{\bS_{\bJ^c}}(P^{\bJ,\bS_{\bJ},\bV_{[T]}})\right)\right) 
  \\
	&\;\; - \log Q^{\bS,\bV_{[T]}}\left( \exp\left(\lambda\left(\hat{\mathcal{R}}_{\bS_{\bJ^c}}-\mathcal{R}_{\mathcal{D}}   -\left(\hat{\mathcal{R}}_{\bS_{\bJ^c}}(Q^{\bS,\bV_{[T]}})-\mathcal{R}_{\mathcal{D}}(Q^{\bS,\bV_{[T]}}) \right)\right)\right)\right).
\end{align*}
On the other hand, since $\ell \in [a_1,a_2]$,  $\lambda\left(\hat{\mathcal{R}}_{\bS_{\bJ^c}}(\bW)-\mathcal{R}_{\mathcal{D}}(\bW)  \right)$ is $\frac{\lambda(a_2-a_1)}{2}$ subgaussian.
Therefore,
\begin{align*}
	&\left(\mathcal{R}_{\mathcal{D}}(Q^{\bS,\bV_{[T]}}) - \hat{\mathcal{R}}_{\bS_{\bJ^c}}(Q^{\bS,\bV_{[T]}})\right) - \left(\mathcal{R}_{\mathcal{D}}(P^{\bJ,\bS_{\bJ},\bV_{[T]}}) - \hat{\mathcal{R}}_{\bS_{\bJ^c}}(P^{\bJ,\bS_{\bJ},\bV_{[T]}})\right)
	\\
		&\leq \inf_{\lambda>0} \frac{\operatorname{KL}(P^{\bJ,\bS_{\bJ},\bV_{[T]}}||Q^{\bS,\bV_{[T]}})  + \frac{1}{8}\lambda^2(a_2-a_1)^2}{\lambda}.
\end{align*}
Since $P^{\bJ,\bS_{\bJ},\bV_{[T]}}$ is independent of $\bS_{\bJ^c}$ then we have $\mathbb{E}^{\bS_{\bJ},\bJ,\bV_{[T]}}\left[\mathcal{R}_{\mathcal{D}}(P^{\bJ,\bS_{\bJ},\bV_{[T]}}) - \hat{\mathcal{R}}_{\bS_{\bJ^c}}(P^{\bJ,\bS_{\bJ},\bV_{[T]}})\right] = 0$.
Hence, by averaging over $\bS_{\bJ^c}$ (equivalently, taking the conditional expectation conditional on $(\bS_{\bJ},\bJ,\bV_{[T]})$) we have, with probability one
\begin{align*}
	&\mathbb{E}^{\bS_{\bJ},\bJ,\bV_{[T]}} \left[ \mathcal{R}_{\mathcal{D}}(Q^{\bS,\bV_{[T]}}) -\hat{\mathcal{R}}_{\bS_{\bJ^c}}(Q^{\bS,\bV_{[T]}})\right]
	\\
	&= \mathbb{E}^{\bS_{\bJ},\bJ,\bV_{[T]}} \left[ \mathcal{R}_{\mathcal{D}}(Q^{\bS,\bV_{[T]}}) -\hat{\mathcal{R}}_{\bS_{\bJ^c}}(Q^{\bS,\bV_{[T]}}) - \left(\mathcal{R}_{\mathcal{D}}(P^{\bJ,\bS_{\bJ},\bV_{[T]}}) - \hat{\mathcal{R}}_{\bS_{\bJ^c}}(P^{\bJ,\bS_{\bJ},\bV_{[T]}})\right)\right]
		\\
	&\leq \mathbb{E}^{\bS_{\bJ},\bJ,\bV_{[T]}} \left( \inf_{\lambda>0} \frac{\operatorname{KL}(P^{\bJ,\bS_{\bJ},\bV_{[T]}}||Q^{\bS,\bV_{[T]}})  + \frac{1}{8}\lambda^2(a_2-a_1)^2}{\lambda}\right)
\end{align*}
Finally, by taking the full expectation, since $J \perp \!\!\! \perp Q^{\bS, \bV_{[T]}}$ we get:
\begin{equation*}
	\mathbb{E}_{\bS,\bV_{[T]}} \left[ \mathcal{R}_{\mathcal{D}}(Q^{\bS, \bV_{[T]}}) - \hat{\mathcal{R}}_{\bS}(Q^{\bS, \bV_{[T]}})\right]
		 \leq \mathbb{E}_{\bS,\bV_{[T]},\bJ} \left[\inf_{\lambda>0} \frac{\operatorname{KL}(P^{\bJ,\bS_{\bJ},\bV_{[T]}}||Q^{\bS,\bV_{[T]}})  + \frac{1}{8}\lambda^2(a_2-a_1)^2}{\lambda}\right]
\end{equation*}
where the final $\operatorname{KL}(P^{\bJ,\bS_{\bJ},\bV_{[T]}}||Q^{\bS,\bV_{[T]}})$ on the right hand side is between two random measures, and hence is a random variable depending on $(\bS,\bJ,\bV_{[T]})$; and the expectation on the right hand side integrates over $(\bS,\bJ,\bV_{[T]})$.

Since 
\begin{equation*}
    \frac{\operatorname{KL}(P^{\bJ,\bS_{\bJ},\bV_{[T]}}||Q^{\bS,\bV_{[T]}})  + \frac{1}{8}\lambda^2(a_2-a_1)^2}{\lambda}\ge \sqrt{\frac{1}{2}(a_2-a_1)^2\operatorname{KL}(P^{\bJ,\bS_{\bJ},\bV_{[T]}}||Q^{\bS,\bV_{[T]}})},
\end{equation*}
the proof is completed.
\end{proof}

\section{Supplementary of Section \ref{sec: a_greedy_approach}}
\label{appen:6}



In this section, we provide the proof of Theorem \ref{thm: greedy}. Specifically, as mentioned in the main body, optimizing $\gen_T$ with greedily selected prior involves three steps.
(1). we first prove Lemma \ref{lem: optimal_one_kl_formal}, which provides the optimal solution of noise covariance and prior for one single KL divergence term in the generalization bound $\gen_T$; (2). as the optimal solution of noise covariance in Lemma \ref{lem: optimal_one_kl_formal} is independent of $\bS_{\bJ}$, $\bV_{[T]}$, and $\bV_{[T]}$, we are then able to obtain the greedy prior by Lemma \ref{lem: form after equivalence}; (3). applying the greedy prior back to $\gen_T$, we are finally able to derive Theorem \ref{thm: greedy}.

We start by restating Lemma \ref{lem: optimal_one_kl} and providing its proof.
\begin{lemma}[Lemma \ref{lem: optimal_one_kl_formal}, restated]
\label{lem: optimal_one_kl}
For any $s\in [T]$, $\bJ$, $\bS_{\bJ}$, and $\bV_{[T]}$, under Constraint \ref{constraint: trace}, 
\begin{equation}
\label{eq: target_one_kl_appen}
\min_{{P}_{s|(s-1)}^{\bJ,\bS_{\bJ},\bV_{s}}}\mathbb{E}_{\bS_{\bJ^c}\sim \mathcal{D}} \operatorname{KL}\left({P}_{s|(s-1)}^{\bJ,\bS_{\bJ},\bV_{s}}\left\Vert Q_{s|(s-1)}^{\bS,\bV_{s}}\right.\right)
\end{equation}
(1). is independent of $\bSig_s$ when $\bV_s\cap \bJ^c=\emptyset$, and (2). is minimized at  $ \bSig_{s}(\bW)=\lambda_s(\bW)\left(\bSig^{pop}_{\bW}\right)^{\frac{1}{2}}$, $\forall \bW$, when $\bV_s\cap \bJ^c\ne\emptyset$, where $\lambda_{s}(\bW)=c_s(\bW)/\tr (\left(\bSig^{pop}_{\bW}\right)^{\frac{1}{2}})$.
\end{lemma}
\begin{proof}
We first calculate $\min_{{P}_{s|(s-1)}^{\bJ,\bS_{\bJ},\bV_{s}}}\mathbb{E}_{\bS_{\bJ^c}\sim \mathcal{D}} \operatorname{KL}\left({P}_{s|(s-1)}^{\bJ,\bS_{\bJ},\bV_{s}}\left\Vert Q_{s|(s-1)}^{\bS,\bV_{s}}\right.\right)$ for any $\bSig_s$. By applying the definition of the KL divergence, we have
\begin{align}
\nonumber
    &\arg \min_{{P}_{s|(s-1)}^{\bJ,\bS_{\bJ},\bV_{s}}}\mathbb{E}_{\bS_{\bJ^c}\sim\mathcal{D}} \operatorname{KL}\left({P}_{s|(s-1)}^{\bJ,\bS_{\bJ},\bV_{s}}\left\Vert Q_{s|(s-1)}^{\bS,\bV_{[T]}}\right.\right)
      \\
\nonumber
      &=\arg \min_{{P}_{s|(s-1)}^{\bJ,\bS_{\bJ},\bV_{s}}}\mathbb{E}_{\bS_{\bJ^c}\sim\mathcal{D}} \int {P}_{s|(s-1)}^{\bJ,\bS_{\bJ},\bV_{s}}(\bW_s)\log\frac{{P}_{s|(s-1)}^{\bJ,\bS_{\bJ},\bV_{s}}(\bW_s)}{Q_{s|(s-1)}^{\bS,\bV_{[T]}}(\bW_s)} \mathrm{d} \bW_s
      \\
\label{eq: log_exp_expect}
      &\overset{(*)}{=} \arg \min_{{P}_{s|(s-1)}^{\bJ,\bS_{\bJ},\bV_{s}}} \int {P}_{s|(s-1)}^{\bJ,\bS_{\bJ},\bV_{s}}(\bW_s)\log\frac{{P}_{s|(s-1)}^{\bJ,\bS_{\bJ},\bV_{s}}(\bW_s)}{e^{\mathbb{E}_{\bS_{\bJ^c}\sim \mathcal{D}}\log Q_{s|(s-1)}^{\bS,\bV_{[T]}}(\bW_s)}} \mathrm{d} \bW_s,
\end{align}
where Eq. $(*)$ is due to the independence of $P$ on $\bS_{\bJ^c}$.

Let 
\begin{align}
\label{eq: def_tilde_Q}
    \tilde{Q}_{s|(s-1)}^{\bJ,\bS_{\bJ},\bV_{[T]}}(\bW)=\frac{e^{\mathbb{E}_{\bS_{\bJ^c}\sim \mathcal{D}}\log Q_{s|(s-1)}^{\bS,\bV_{[T]}}(\bW)}}{\int e^{\mathbb{E}_{\bS_{\bJ^c}\sim \mathcal{D}}\log Q_{s|(s-1)}^{\bS,\bV_{[T]}}(\tilde{\bW})} \mathrm{d} \tilde{\bW}},
\end{align}
and $\tilde{Q}_{s|(s-1)}^{\bJ,\bS_{\bJ},\bV_{[T]}}(\bW)$ is then a probability measure on $\mathbb{R}^d$. Applying Eq. (\ref{eq: def_tilde_Q}) back to Eq. (\ref{eq: log_exp_expect}), we obtain
\begin{small}
\begin{align}
\nonumber
    &\arg \min_{{P}_{s|(s-1)}^{\bJ,\bS_{\bJ},\bV_{s}}}\left( \int {P}_{s|(s-1)}^{\bJ,\bS_{\bJ},\bV_{s}}(\bW_s)\log\frac{{P}_{s|(s-1)}^{\bJ,\bS_{\bJ},\bV_{s}}(\bW_s)}{\tilde{Q}_{s|(s-1)}^{\bJ,\bS_{\bJ},\bV_{[T]}}(\bW_s)} \mathrm{d} \bW_s \right.
    \\
    \nonumber
    &~~~~~~~~~~~~~~~~~~~~~~~~~~~~~~~~~~~~
    -  \left.\int {P}_{s|(s-1)}^{\bJ,\bS_{\bJ},\bV_{s}}(\bW_s)\log\left( \int e^{\mathbb{E}_{\bS_{\bJ^c}\sim \mathcal{D}}\log Q_{s|(s-1)}^{\bS,\bV_{[T]}}(\tilde{\bW})} \mathrm{d} \tilde{\bW}\right) \mathrm{d} \bW_s\right)
    \\
\nonumber
    =& \arg \min_{{P}_{s|(s-1)}^{\bJ,\bS_{\bJ},\bV_{s}}}\left( \int {P}_{s|(s-1)}^{\bJ,\bS_{\bJ},\bV_{s}}(\bW_s)\log\frac{{P}_{s|(s-1)}^{\bJ,\bS_{\bJ},\bV_{s}}(\bW_s)}{\tilde{Q}_{s|(s-1)}^{\bJ,\bS_{\bJ},\bV_{[T]}}(\bW_s)} \mathrm{d} \bW_s \right.
    - \left.\log\left( \int e^{\mathbb{E}_{\bS_{\bJ^c}\sim \mathcal{D}}\log Q_{s|(s-1)}^{\bS,\bV_{[T]}}(\tilde{\bW})} \mathrm{d} \tilde{\bW}\right)\right)
    \\
\nonumber
     =& \arg \min_{{P}_{s|(s-1)}^{\bJ,\bS_{\bJ},\bV_{s}}}\left( \int {P}_{s|(s-1)}^{\bJ,\bS_{\bJ},\bV_{s}}(\bW_s)\log\frac{{P}_{s|(s-1)}^{\bJ,\bS_{\bJ},\bV_{s}}(\bW_s)}{\tilde{Q}_{s|(s-1)}^{\bJ,\bS_{\bJ},\bV_{[T]}}(\bW_s)} \mathrm{d} \bW_s \right)
     \\
    \label{eq: kl_minimum}
=&\arg \min_{{P}_{s|(s-1)}^{\bJ,\bS_{\bJ},\bV_{s}}}\operatorname{KL}\left(  P \Vert \tilde{Q}_{s|(s-1)}^{\bJ,\bS_{\bJ},\bV_{[T]}} \right).
\end{align}
\end{small}
The minimum of Eq.(\ref{eq: kl_minimum}) is achieved if and only if ${P}_{s|(s-1)}^{\bJ,\bS_{\bJ},\bV_{s}}= \tilde{Q}_{s|(s-1)}^{\bJ,\bS_{\bJ},\bV_{[T]}}$, and we only need to calculate the exact form of $\tilde{Q}_{s|(s-1)}^{\bJ,\bS_{\bJ},\bV_{[T]}}$. Since $\bW_s|(\bW_{s-1},\bS, \bV_s)\sim \mathcal{N}(\bW_{s-1}-\eta_s \nabla_{\bW_{s-1}} \mathcal{R}_{\bS_{\bV_s}}(\bW_{s-1}),$ $ \bSig_s(\bW_{s-1}))$, we have
\begin{small}
\begin{align}
\nonumber
    &\exp{\mathbb{E}_{\bS_{\bJ^c}\sim \mathcal{D}}\log Q_{s|(s-1)}^{\bS,\bV_{[T]}}(\bW)}
    \\
\nonumber
    &= \exp\left(\mathbb{E}_{\bS_{\bJ^c}\sim \mathcal{D}}\left(-\frac{1}{2}(\bW-\bW_{s-1}+\eta_s \nabla_{\bW_{s-1}} \mathcal{R}_{\bS_{\bV_s}}(\bW_{s-1}))^\top\bSig_s(\bW_{s-1})^{-1}(\bW-\bW_{s-1}\right. \right.
    \\
\nonumber
    &\;\;\left.\left.+\eta_s \nabla_{\bW_{s-1}} \mathcal{R}_{\bS_{\bV_s}}(\bW_{s-1}))-\frac{d}{2}\log 2\pi-\frac{1}{2}\log \det (\bSig_s(\bW_{s-1}))\right)\right)
    \\\nonumber
    &=\exp\left(\mathbb{E}_{\bS_{\bJ^c}\sim \mathcal{D}}\bigg(-\frac{1}{2}(\bW-\bW_{s-1}+\eta_s \nabla_{\bW_{s-1}} \mathcal{R}_{\bS_{\bV_s}}(\bW_{s-1}))^\top\bSig_s(\bW_{s-1})^{-1}(\bW-\bW_{s-1}\right. 
    \\
    \label{eq: exponential_Q}
    &\;\;\left.+\eta_s \nabla_{\bW_{s-1}} \mathcal{R}_{\bS_{\bV_s}}(\bW_{s-1})\bigg)-\frac{d}{2}\log 2\pi-\frac{1}{2}\log \det (\bSig_s(\bW_{s-1}))\right).
\end{align}
\end{small}
On the other hand,
\begin{small}
\begin{align}
\nonumber
    &\mathbb{E}_{\bS_{\bJ^c}\sim \mathcal{D}}\left(-\frac{1}{2}(\bW-\bW_{s-1}+\eta_s \nabla_{\bW_{s-1}} \mathcal{R}_{\bS_{\bV_s}}(\bW_{s-1}))^\top\bSig_s(\bW_{s-1})^{-1}(\bW-\bW_{s-1}\right.
    \\
\nonumber
    &\;\;\left.+\eta_s \nabla_{\bW_{s-1}} \mathcal{R}_{\bS_{\bV_s}}(\bW_{s-1}))\right)
    \\
    \nonumber
    &=-\frac{1}{2}\mathbb{E}_{\bS_{\bJ^c}\sim \mathcal{D}}\left(\bW-\bW_{s-1}+\eta_s \left(\frac{\vert\bV_s \cap \bJ \vert}{\vert\bV_s \vert}\nabla \mathcal{R}_{\bS_{\bV_s \cap \bJ} }\left(\bW_{s-1}\right)+\frac{\vert\bV_s\cap \bJ^c\vert}{\vert\bV_s\vert }\nabla \mathcal{R}_{\bS_{\bV_s \cap \bJ^c}}\left(\bW_{s-1}\right)\right)\right)^\top
    \\
\nonumber
    &\;\;\cdot \bSig_s(\bW_{s-1})^{-1}\left(\bW-\bW_{s-1}+\eta_s \left(\frac{\vert\bV_s \cap \bJ \vert}{\vert\bV_s \vert}\nabla \mathcal{R}_{\bS_{\bV_s \cap \bJ} }\left(\bW_{s-1}\right)
    +\frac{\vert\bV_s\cap \bJ^c\vert}{\vert\bV_s\vert }\nabla \mathcal{R}_{\bS_{\bV_s \cap \bJ^c} }\left(\bW_{s-1}\right)\right)\right)
    \\
\nonumber
    &=-\frac{1}{2}\left(\bW-\bW_{s-1}+\eta_s \left(\frac{\vert\bV_s \cap \bJ \vert}{\vert\bV_s \vert}\nabla \mathcal{R}_{\bS_{\bV_s \cap \bJ} }\left(\bW_{s-1}\right)+\frac{\vert\bV_s\cap \bJ^c\vert}{\vert\bV_s\vert }\nabla \mathcal{R}_{\mathcal{D} }\left(\bW_{s-1}\right)\right)\right)^\top
    \\
\nonumber
    &\;\;\cdot \bSig_s(\bW_{s-1})^{-1}\left(\bW-\bW_{s-1}+\eta_s \left(\frac{\vert\bV_s \cap \bJ \vert}{\vert\bV_s \vert}\nabla \mathcal{R}_{\bS_{\bV_s \cap \bJ} }\left(\bW_{s-1}\right)+\frac{\vert\bV_s\cap \bJ^c\vert}{\vert\bV_s\vert }\nabla \mathcal{R}_{\mathcal{D} }\left(\bW_{s-1}\right)\right)\right)
    \\
    \nonumber
    &\;\;-\frac{1}{2}\mathbb{E}_{\bS_{\bJ^c}\sim \mathcal{D}}\eta_s^2\frac{\vert \bV_s \cap \bJ^c \vert^2}{\vert \bV_s\vert^2 }\left(\nabla \mathcal{R}_{\mathcal{D} }\left(\bW_{s-1}\right)-\nabla \mathcal{R}_{\bS_{\bV_s\cap \bJ^c} }\left(\bW_{s-1}\right)\right)^\top\bSig_s(\bW_{s-1})^{-1}
    \\
\label{eq: expectation_Q}
    &\;\;\cdot\left(\nabla \mathcal{R}_{\mathcal{D} }\left(\bW_{s-1}\right)-\nabla \mathcal{R}_{\bS_{\bV_s\cap \bJ^c} }\left(\bW_{s-1}\right)\right).
\end{align}
\end{small}
By combining Eq.(\ref{eq: exponential_Q}) and Eq.(\ref{eq: expectation_Q}), we further have
\begin{small}
\begin{align}
\nonumber
    &\exp{\mathbb{E}_{\bS_{\bJ^c}\sim \mathcal{D}}\log Q_{s|(s-1)}^{\bS,\bV_{[T]}}(\bW)}
    \\
\nonumber
    &=\frac{1}{(2\pi)^{-\frac{d}{2}}\det(\bSig_s(\bW_{s-1}))^{\frac{1}{2}}}\exp\left(-\frac{1}{2}\left(\bW-\bW_{s-1}+\eta_s \left(\frac{\vert\bV_s \cap \bJ \vert}{\vert\bV_s \vert}\nabla \mathcal{R}_{\bS_{\bV_s \cap \bJ} }\left(\bW_{s-1}\right)\right.\right.\right.
    \\
\nonumber
    &\;\;+\left.\left.\left.\frac{\vert\bV_s\cap \bJ^c\vert}{\vert\bV_s\vert }\nabla \mathcal{R}_{\mathcal{D} }\left(\bW_{s-1}\right)\right)\right)^\top\bSig_s(\bW_{s-1})^{-1}\left(\bW-\bW_{s-1}+\eta_s\left(\frac{\vert\bV_s \cap \bJ \vert}{\vert\bV_s \vert}\nabla \mathcal{R}_{\bS_{\bV_s \cap \bJ} }\left(\bW_{s-1}\right)\right.\right.\right.
    \\
\nonumber
    &\;\; \left.\left.+\frac{\vert\bV_s\cap \bJ^c\vert}{\vert\bV_s\vert }\nabla \mathcal{R}_{\mathcal{D} }\left(\bW_{s-1}\right)\right)\right) \exp\mathbb{E}_{\bS_{\bJ^c}}\bigg(-\frac{1}{2}\eta_s^2\frac{\vert \bV_s \cap \bJ^c \vert^2}{\vert \bV_s\vert^2 }\left(\nabla \mathcal{R}_{\mathcal{D} }\left(\bW_{s-1}\right)-\nabla \mathcal{R}_{\bS_{\bV_s\cap \bJ^c} }\left(\bW_{s-1}\right)\right)^\top
    \\
    \label{eq: expect_Q}
    &\;\;
  \cdot\bSig_s(\bW_{s-1})^{-1}\left(\nabla \mathcal{R}_{\mathcal{D} }\left(\bW_{s-1}\right)-\nabla \mathcal{R}_{\bS_{\bV_s\cap \bJ^c} }\left(\bW_{s-1}\right)\right)\bigg).
\end{align}
\end{small}
Therefore, by taking integration with respect to $\tilde{\bW}$, we have,
\begin{align}
\nonumber
    &\int e^{\mathbb{E}_{\bS_{\bJ^c}\sim \mathcal{D}}\log Q_{s|(s-1)}^{\bS,\bV_{[T]}}(\tilde{\bW})} \mathrm{d} \tilde{\bW}
    \\
\nonumber
    &=\exp\mathbb{E}_{\bS_{\bJ^c}}\left(-\frac{1}{2}\eta_s^2\frac{\vert \bV_s \cap \bJ^c \vert^2}{\vert \bV_s\vert^2 }\left(\nabla \mathcal{R}_{\mathcal{D} }\left(\bW_{s-1}\right)-\nabla \mathcal{R}_{\bS_{\bV_s\cap \bJ^c} }\left(\bW_{s-1}\right)\right)^\top\bSig_s(\bW_{s-1})^{-1}
    \right.
    \\
\label{eq: expect_Q_int}
   &\;\;\left.\cdot\left(\nabla \mathcal{R}_{\mathcal{D} }\left(\bW_{s-1}\right)-\nabla \mathcal{R}_{\bS_{\bV_s\cap \bJ^c} }\left(\bW_{s-1}\right)\right)\right).
\end{align}

Therefore, by Eq.(\ref{eq: def_tilde_Q}), Eq.(\ref{eq: expect_Q}), and Eq.(\ref{eq: expect_Q_int}), we have
\begin{align}
    \label{eq: form_greedy_prior}
     &\arg \min_{{P}_{s|(s-1)}^{\bJ,\bS_{\bJ},\bV_{s}}}\mathbb{E}_{\bS_{\bJ^c}\sim\mathcal{D}} \operatorname{KL}\left({P}_{s|(s-1)}^{\bJ,\bS_{\bJ},\bV_{s}}\left\Vert Q_{s|(s-1)}^{\bS,\bV_{[T]}}\right.\right)=\tilde{Q}_{s|(s-1)}^{\bJ,\bS_{\bJ},\bV_{[T]}}
     \\
     \nonumber
     &\sim \mathcal{N}\left(\bW_{s-1}-\eta_s \left(\frac{\vert\bV_s \cap \bJ \vert}{\vert\bV_s \vert}\nabla \mathcal{R}_{\bS_{\bV_s \cap \bJ} }\left(\bW_{s-1}\right)+\frac{\vert\bV_s\cap \bJ^c\vert}{\vert\bV_s\vert }\nabla \mathcal{R}_{\mathcal{D} }\left(\bW_{s-1}\right)\right),\bSig_s(\bW_{s-1})\right).
\end{align}
Applying Eq. (\ref{eq: form_greedy_prior}) back to $\mathbb{E}_{\bS_{\bJ^c}\sim\mathcal{D}} \operatorname{KL}\left({P}_{s|(s-1)}^{\bJ,\bS_{\bJ},\bV_{s}}\left\Vert Q_{s|(s-1)}^{\bS,\bV_{[T]}}\right.\right)$, we obtain
\begin{align}
\nonumber
    &\min_{{P}_{s|(s-1)}^{\bJ,\bS_{\bJ},\bV_{s}}}\mathbb{E}_{\bS_{\bJ^c}\sim\mathcal{D}} \operatorname{KL}\left({P}_{s|(s-1)}^{\bJ,\bS_{\bJ},\bV_{s}}\left\Vert Q_{s|(s-1)}^{\bS,\bV_{[T]}}\right.\right)=\mathbb{E}_{\bS_{\bJ^c}\sim\mathcal{D}} \operatorname{KL}\left(\tilde{Q}_{s|(s-1)}^{\bJ,\bS_{\bJ},\bV_{[T]}}\left\Vert Q_{s|(s-1)}^{\bS,\bV_{[T]}}\right.\right)
    \\
    \nonumber
    &=\int \tilde{Q}_{t|(t-1)}^{\bS_{\bJ},\bV_{[s]}}(\bW_s)\log\frac{\tilde{Q}_{t|(t-1)}^{\bS_{\bJ},\bV_{[s]}}(\bW_s)}{e^{\mathbb{E}_{\bS_{\bJ^c}\sim \mathcal{D}}\log Q_{t|(t-1)}^{\bS,\bV_{[s]}}(\bW_s)}} \mathrm{d} \bW_s
    \\
\nonumber
    &\overset{(\circ)}{=}-\int \tilde{Q}_{t|(t-1)}^{\bS_{\bJ},\bV_{[s]}}(\bW_s)\log\int e^{\mathbb{E}_{\bS_{\bJ^c}\sim \mathcal{D}}\log Q_{t|(t-1)}^{\bS,\bV_{[s]}}(\tilde{\bW})} \mathrm{d} \tilde{\bW} \mathrm{d} \bW_s
    \\
\nonumber
    &=-\log\int e^{\mathbb{E}_{\bS_{\bJ^c}\sim \mathcal{D}}\log Q_{t|(t-1)}^{\bS,\bV_{[s]}}(\tilde{\bW})} \mathrm{d} \tilde{\bW}
    \\
\nonumber
    &\overset{(\bullet)}{=}\frac{1}{2}\eta_t^2\mathbb{E}_{\bS_{\bJ^c}\sim \mathcal{D}}\frac{\vert \bV_t \cap \bJ^c\vert^2}{\vert \bV_t\vert^2 }\left(\nabla \mathcal{R}_{\mathcal{D} }\left(\bW_{t-1}\right)-\nabla \mathcal{R}_{\bS_{\bV_t\cap \bJ^c} }\left(\bW_{t-1}\right)\right)^\top\bSig_t(\bW_{t-1})^{-1}
    \\
\nonumber
   &\;\;\cdot\left(\nabla \mathcal{R}_{\mathcal{D} }\left(\bW_{t-1}\right)-\nabla \mathcal{R}_{\bS_{\bV_t\cap \bJ^c} }\left(\bW_{t-1}\right)\right)
   \\
   \nonumber
   &=\frac{1}{2}\eta_t^2\mathbb{E}_{\bS_{\bJ^c}\sim \mathcal{D}}\tr\bigg(\bSig_t(\bW_{t-1})^{-1}\frac{\vert \bV_t \cap \bJ^c\vert^2}{\vert \bV_t\vert^2 }\left(\nabla \mathcal{R}_{\mathcal{D} }\left(\bW_{t-1}\right)-\nabla \mathcal{R}_{\bS_{\bV_t\cap \bJ^c} }\left(\bW_{t-1}\right)\right)^\top
    \\
\nonumber
   &\;\;\cdot\left(\nabla \mathcal{R}_{\mathcal{D} }\left(\bW_{t-1}\right)-\nabla \mathcal{R}_{\bS_{\bV_t\cap \bJ^c} }\left(\bW_{t-1}\right)\right)\bigg)
   \\
  \label{eq:greedy_kl}
   &\overset{(\Diamond)}{=}
\left\{  
             \begin{aligned}
&0 \qquad\qquad\qquad\qquad\qquad\qquad\qquad\qquad, \bV_t\cap\bJ^c=\emptyset; \\  
  &\frac{1}{2}\frac{\eta_t^2N}{b_t(N-1)^2}\tr\bigg(\bSig_t(\bW_{t-1})^{-1} \bSig^{pop}_{\bW_{t-1}}\bigg)          ,     \bV_t\cap\bJ^c\ne\emptyset  .    \end{aligned} 
\right.  
\end{align}
where  Eq. ($\circ$) is due to the definition of $\tilde{Q}_{t|(t-1)}^{\bS_{\bJ},\bV_{[s]}}$ (Eq.(\ref{eq: def_tilde_Q})), Eq. ($\bullet$) is due to Eq.(\ref{eq: expect_Q_int}) and Eq. ($\Diamond$) is due to Lemma \ref{lem: sampling}.

Therefore, when $\bV_t\cap\bJ^c=\emptyset$, Eq.(\ref{eq: target_one_kl_appen}) is independent of $\bSig_s$. On the other hand, if $\bV_t\cap\bJ^c\ne\emptyset$, we only need to solve 
\begin{equation}
\label{eq: optimization_trace_sqrt}
    \bSig_s(\bW)^{*}=\arg\min_{\tr(\bSig_s(\bW))=c_s(\bW)}  \tr\bigg(\bSig_s(\bW)^{-1} \bSig^{pop}_{\bW}\bigg), \text{ subject to Constraint \ref{constraint: trace}}.
\end{equation}

We complete the proof by solving Problem (\ref{eq: optimization_trace_sqrt}). Specifically, let the eigenvalues of $\bSig_{\bW}^{pop}$ be $(\omega_i^{pop})_{i=1}^d$ (the value is by non-increasing order with respect to index) we first fix the eigenvalues of $\bSig_s(\bW)$ to be $\alpha_{[d]}$ with $\alpha_{i}\ge 0$ (the value is by non-increasing order with respect to index), $i\in[d]$. Then, by Lemma \ref{lem: ortho_matrix}, the minimum of $\tr\bigg(\bSig_s(\bW)^{-1} \bSig^{pop}_{\bW}\bigg)$ is achieved when 
\begin{equation}
\label{eq:ortho_sqrt}
    \bSig_s(\bW)\in\left\{P^\top\left(\alpha_{[d]}\right)P:P\text{ is orthogonal and }\bSig^{pop}_{\bW}=P^\top\left(\omega_{[d]}^{pop}\right)P\right\},
\end{equation}
and 
\begin{equation*}
     \tr\bigg(\bSig_s(\bW)^{-1} \bSig^{pop}_{\bW}\bigg)=\sum_{i=1}^d \frac{\omega_{i}^{pop}}{\alpha_i}.
\end{equation*}
We then optimize $\sum_{i=1}^d \frac{\omega_{i}^{pop}}{\alpha_i}$ under the constraint $\sum_{i=1}^d \alpha_i=c_s(\bW_{s-1})$. By the Cauchy-Schwarz
inequality,
\begin{equation}
\label{eq:value_sqrt}
     c_s(\bW_{s-1})\left(\sum_{i=1}^d \frac{\omega_{i}^{pop}}{\alpha_i}\right)=\left(\sum_{i=1}^d \frac{\omega_{i}^{pop}}{\alpha_i}\right)\left(\sum_{i=1}^d \alpha_i\right)\overset{(*)}{\ge} \left(\sum_{i=1}^d \sqrt{\omega_{i}^{pop}}\right)^2,
\end{equation}
where equality in inequality $(*)$ holds when $\alpha_i^2/\omega^{pop}_i$ is invariant of $i$. By combining Eq.(\ref{eq:ortho_sqrt}) and Eq.(\ref{eq:value_sqrt}), the proof is completed.
\end{proof}

By Lemma \ref{lem: optimal_one_kl}, the optimal noise covariances $\bSig_s$ of all KL divergence terms $\mathbb{E}_{\bS_{\bJ^c}\sim\mathcal{D}} \operatorname{KL}\left({P}_{s|(s-1)}^{\bJ,\bS_{\bJ},\bV_{s}}\left\Vert Q_{s|(s-1)}^{\bS,\bV_{s}}\right.\right)$ are the same regardless of $\bV_s$, $\bJ$, and $\bS_{\bJ}$, which helps us to obtain Lemma \ref{lem: form after equivalence}.

\begin{proof}[Proof of Lemma \ref{lem: form after equivalence}]
To begin with, denote the optimal noise covariance of first $s$-step in terms of the generalization bound $\gen_{s}$ as $\bSig^{s}$ under Constraint \ref{constraint: trace}, i.e.,
\begin{equation*}
\bSig^s_{[s]} \overset{\triangle}{=}\arg \min_{\bSig_{[s]}}\left(\min_{P}\gen_s(P, \bSig_{[s]})\right), \text{ subject to: Constraint \ref{constraint: trace}},
\end{equation*}
we also define $Q^s$ accordingly as the posterior distribution with noise covariance $\bSig^s$.
 Also, recall that $P^s$ is the optimal prior in terms of the generalization bound $\gen_{s}$ under Constraint \ref{constraint: trace}, i.e., 
\begin{equation*}
    P^{s}=\arg \min_{P}\left(\min_{\bSig_{[s]}}\gen_s(P, \bSig_{[s]})\right), \text{ subject to: Constraint \ref{constraint: trace}}.
\end{equation*}
We would like to derive the form of $\bSig_s^s$ and $P_{s|(s-1)}^s$.

Specifically, we have 
\begin{align*}
    P^s_{s|(s-1)}=\operatornamewithlimits{\arg\min}_{P_{s|(s-1)}}\left( \gen_{s}(P, \bSig^s_{[s]})\right), \text{ subject to: $P_{t|(t-1)}=P^s_{t|(t-1)} (t<s)$},
\end{align*}
and 
\begin{align*}
    \bSig^s_{s}=\operatornamewithlimits{\arg\min}_{\bSig_{s}}\left( \gen_{s}(P^s, \bSig_{[s]})\right),\text{ subject to: Constraint \ref{constraint: trace} and $\bSig_{t}=\bSig^s_{t|(t-1)} (t<s)$}.
\end{align*}
That is, to obtain the desired $\bSig^s_s$ and $P^{s}_{s|(s-1)}$, we only need to solve 
\begin{equation*}
    \min_{\bSig_s,P^s_{s|(s-1)}} \gen_s(P, \bSig_{[s]}), \text{ subject to: $P_{t|(t-1)}=P^s_{t|(t-1)} (t<s)$ and $\bSig_{t}=\bSig^s_{t|(t-1)} (t<s)$}.
\end{equation*}
On the other hand, with $P_{t|(t-1)}=P^s_{t|(t-1)} (t<s)$ and $\bSig_{t}=\bSig^s_{t|(t-1)} (t<s)$ and under Constraint \ref{constraint: trace}, we have 
\begin{align*}
      &\min_{\bSig_s,P_{s|(s-1)}} \gen_s(P, \bSig_{[s]})
      \\
     & = \min_{\bSig_s,P_{s|(s-1)}}\mathbb{E}_{\bS_{\bJ},\bV_{[s]},\bJ}\sqrt{\frac{(a_2-a_1)^2}{2}\mathbb{E}_{\bS_{\bJ^c}}\operatorname{KL}\left(P^{\bJ,\bS_{\bJ},\bV_{[s]}}\left\Vert Q^{\bS,\bV_{[s]}}\right.\right)}
      \\
      &= \min_{\bSig_s,P_{s|(s-1)}}\mathbb{E}_{\bS_{\bJ},\bV_{[s]},\bJ}\sqrt{\frac{(a_2-a_1)^2}{2}\mathbb{E}_{\bS_{\bJ^c}}\sum_{t=1}^s \mathbb{E}_{P_{t-1}^{\bJ, \bS_{\bJ},\bV_{[s]}}} \operatorname{KL}\left(P_{t|(t-1)}^{\bJ,\bS_{\bJ},\bV_{s}}\left\Vert Q_{t|(t-1)}^{\bS,\bV_{s}}\right.\right)}
      \\
      &=\min_{\bSig_s,P_{s|(s-1)}}\mathbb{E}_{\bS_{\bJ},\bV_{[s]},\bJ}\left[\sqrt{\frac{(a_2-a_1)^2}{2}\mathbb{E}_{\bS_{\bJ^c}}\sum_{t=1}^s \mathbb{E}_{{P^s}_{t-1}^{\bJ, \bS_{\bJ},\bV_{[s]}}} \operatorname{KL}\left({P^s}_{t|(t-1)}^{\bJ,\bS_{\bJ},\bV_{s}}\left\Vert {Q^s}_{t|(t-1)}^{\bS,\bV_{s}}\right.\right)}\right.
      \\
      &\;\;\left.\overline{+\frac{(a_2-a_1)^2}{2}\mathbb{E}_{\bS_{\bJ^c}} \mathbb{E}_{P_{s-1}^{\bJ, \bS_{\bJ},\bV_{[s]}}} \operatorname{KL}\left(P_{s|(s-1)}^{\bJ,\bS_{\bJ},\bV_{s}}\left\Vert Q_{s|(s-1)}^{\bS,\bV_{s}}\right.\right)}\right]
      \\
      &\overset{(*)}{\ge} \mathbb{E}_{\bS_{\bJ},\bV_{[s]},\bJ}\left[\sqrt{\frac{(a_2-a_1)^2}{2}\mathbb{E}_{\bS_{\bJ^c}}\sum_{t=1}^s \mathbb{E}_{{P^s}_{t-1}^{\bJ, \bS_{\bJ},\bV_{[s]}}} \operatorname{KL}\left({P^s}_{t|(t-1)}^{\bJ,\bS_{\bJ},\bV_{s}}\left\Vert {Q^s}_{t|(t-1)}^{\bS,\bV_{s}}\right.\right)}\right.
      \\
      &\;\;\left.\overline{+\frac{(a_2-a_1)^2}{2}\mathbb{E}_{\bS_{\bJ^c}} \mathbb{E}_{P_{s-1}^{\bJ, \bS_{\bJ},\bV_{[s]}}} \min_{\bSig_s,P^{\bJ,\bS_{\bJ},\bV_s}_{s|(s-1)}}\operatorname{KL}\left(P_{s|(s-1)}^{\bJ,\bS_{\bJ},\bV_{s}}\left\Vert Q_{s|(s-1)}^{\bS,\bV_{s}}\right.\right)}\right].
\end{align*}
By Lemma \ref{lem: optimal_one_kl}, $\min_{\bSig_s,P^{\bJ,\bS_{\bJ},\bV_s}_{s|(s-1)}}\operatorname{KL}\left(P_{s|(s-1)}^{\bJ,\bS_{\bJ},\bV_{s}}\left\Vert Q_{s|(s-1)}^{\bS,\bV_{s}}\right.\right)$ is attained at $ \bSig_{s}(\bW)=\lambda_s(\bW)\left(\bSig^{pop}_{\bW}\right)^{\frac{1}{2}}$, which is not dependent on $\bJ,\bS_{\bJ},\bV_s$, and 
\begin{small}
\begin{equation*}
    P^{\bJ,\bS_{\bJ},\bV_s}_{s|(s-1)}\sim\mathcal{N}\left(\bW_{s-1}-\eta_s \left(\frac{\vert\bV_s \cap \bJ \vert}{\vert\bV_s \vert}\nabla \mathcal{R}_{\bS_{\bV_s \cap \bJ} }\left(\bW_{s-1}\right)+\frac{\vert\bV_s\cap \bJ^c\vert}{\vert\bV_s\vert }\nabla \mathcal{R}_{\mathcal{D} }\left(\bW_{s-1}\right)\right),\lambda_s(\bW)\left(\bSig^{pop}_{\bW}\right)^{\frac{1}{2}}\right).
\end{equation*}
\end{small}
Therefore, Inequality $(*)$ holds, and the proof is completed.

\end{proof}

By Lemma \ref{lem: form after equivalence}, we obtain the form of $P^*$, i.e., 
\begin{small}
\begin{equation*}
     {P^*}^{\bJ,\bS_{\bJ},\bV_s}_{s|(s-1)}\sim\mathcal{N}\left(\bW_{s-1}-\eta_s \left(\frac{\vert\bV_s \cap \bJ \vert}{\vert\bV_s \vert}\nabla \mathcal{R}_{\bS_{\bV_s \cap \bJ} }\left(\bW_{s-1}\right)+\frac{\vert\bV_s\cap \bJ^c\vert}{\vert\bV_s\vert }\nabla \mathcal{R}_{\mathcal{D} }\left(\bW_{s-1}\right)\right),\lambda_s(\bW)\left(\bSig^{pop}_{\bW}\right)^{\frac{1}{2}}\right),
\end{equation*}
\end{small}
which allows us to further derive Theorem \ref{thm: greedy}.

\begin{proof}[Proof of Theorem \ref{thm: greedy}]
By the definition of $\gen_T$, with prior the greedy prior and under Constraint \ref{constraint: trace}, we have 
\begin{align*}
&\min_{\bSig_{[T]}} \gen_{T}(P^*,\bSig_{[T]})
\\
&=\min_{\bSig_{[T]}}\mathbb{E}_{\bS_{\bJ},\bV_{[T]},\bJ}\sqrt{\frac{(a_2-a_1)^2}{2}\mathbb{E}_{\bS_{\bJ^c}}\operatorname{KL}\left({P^*}^{\bJ,\bS_{\bJ},\bV_{[T]}}\left\Vert Q^{\bS,\bV_{[T]}}\right.\right)}
\\
&= \min_{\bSig_{[T]}}\mathbb{E}_{\bS_{\bJ},\bV_{[T]},\bJ}\sqrt{\frac{(a_2-a_1)^2}{2}\mathbb{E}_{\bS_{\bJ^c}}\sum_{t=1}^T \mathbb{E}_{{P^*}_{t-1}^{\bJ, \bS_{\bJ},\bV_{[t-1]}}} \operatorname{KL}\left({P^*}_{t|(t-1)}^{\bJ,\bS_{\bJ},\bV_{t}}\left\Vert Q_{t|(t-1)}^{\bS,\bV_{t}}\right.\right)}
\\
&\overset{(\bullet)}{\ge} \mathbb{E}_{\bS_{\bJ},\bV_{[T]},\bJ}\sqrt{\frac{(a_2-a_1)^2}{2}\sum_{t=1}^T \mathbb{E}_{{P^*}_{t-1}^{\bJ, \bS_{\bJ},\bV_{[t-1]}}}\min_{\bSig_{t}}\mathbb{E}_{\bS_{\bJ^c}} \operatorname{KL}\left({P^*}_{t|(t-1)}^{\bJ,\bS_{\bJ},\bV_{t}}\left\Vert Q_{t|(t-1)}^{\bS,\bV_{t}}\right.\right)}
\\
&\overset{(*)}{=}\mathbb{E}_{\bS_{\bJ},\bV_{[T]},\bJ}\sqrt{\frac{(a_2-a_1)^2}{2}\sum_{t=1}^T \mathbb{E}_{{P^*}_{t-1}^{\bJ, \bS_{\bJ},\bV_{[t-1]}}}\min_{\bSig_{t},P_{t|(t-1)}^{\bJ,\bS_{\bJ},\bV_{t}}}\mathbb{E}_{\bS_{\bJ^c}} \operatorname{KL}\left(P_{t|(t-1)}^{\bJ,\bS_{\bJ},\bV_{t}}\left\Vert Q_{t|(t-1)}^{\bS,\bV_{t}}\right.\right)},
\end{align*}
where Eq. $(*)$ is due to that by the proof of Lemma \ref{lem: form after equivalence}, $ {P^*}^{\bJ,\bS_{\bJ},\bV_{s}}_{s|(s-1)}$ is the same as the prior minimizing $\mathbb{E}_{\bS_{\bJ^c}\sim \mathcal{D}} \operatorname{KL}\left({P}_{s|(s-1)}^{\bJ,\bS_{\bJ},\bV_{s}}\left\Vert Q_{s|(s-1)}^{\bS,\bV_{s}}\right.\right)$ for any given ${\bJ,\bS_{\bJ},\bV_{s}}$. Therefore, $\min_{\bSig_{t}}\mathbb{E}_{\bS_{\bJ^c}} \operatorname{KL}\left({P^*}_{t|(t-1)}^{\bJ,\bS_{\bJ},\bV_{t}}\left\Vert Q_{t|(t-1)}^{\bS,\bV_{t}}\right.\right)$ is attained when $\bSig_t(\bW)=\lambda_s(\bW)(\bSig^{pop}_{\bW})^{\frac{1}{2}}$, which is independent of ${\bJ,\bS_{\bJ},\bV_{s}}$, and Inequality ($\bullet$) holds. Therefore, $\min_{\bSig_{[T]}} \gen_{T}(P^*,\bSig_{[T]})$ is also attained at $\bSig_t(\bW)=\lambda_s(\bW)(\bSig^{pop}_{\bW})^{\frac{1}{2}}$.

The proof is completed.

\end{proof}

\section{Supplementary materials of Section \ref{sec: optimization_fixed_prior}}
\label{appen:5}
\subsection{Formal Description of the Prior in Section \ref{sec: optimization_fixed_prior}}
\label{sec: descri_algorithm}
In this section, we provide a detailed description of the update rule of the prior defined by Eq.(\ref{eq: update_prior}).

\begin{algorithm}[H]
\DontPrintSemicolon
  \label{alg: prior}
  \KwInput{Sample set $\bS$ with size $N$, initialization distribution $\mathcal{W}_0$, total step $T$, learning rate $(\eta_t)_{t=1}^T$}
  \KwOutput{$\bW_{[T]}$,  $\bJ$}
 Initialize $\bW_0$ according to $\mathcal{W}_0$; initialize $\bJ$ by uniformly sampling $N-1$ elements from $[N]$ without replacement; set $t=0$\;
 \While{$t<T$}{
  Uniformly sample index set $\bV_t\subset [N]$ such that $\vert \bV_t \vert=b_t$ without replacement and independent of $\bJ$\; 
  \eIf{$\bV_t\subset \bJ$}{
  $\bW_{t}=\bW_{t-1}-\eta_t \nabla \mathcal{R}_{\bS_{\bV_t}}(\bW_{t-1})+\mathcal{N}\left(\boldsymbol{0},\sigma_t \mathbb{I}_d\right)$\;
   }{
  $\bW_{t}=\bW_{t-1}-\eta_t \frac{b_t-1}{b_t}\nabla \mathcal{R}_{\bS_{\bV_t\cap\bJ}}(\bW_{t-1})-\eta_t \frac{1}{b_t}\nabla \mathcal{R}_{\bS_{\bJ}}(\bW_{t-1})+\mathcal{N}\left(\boldsymbol{0},\sigma_t \mathbb{I}_d\right)$\;
  }
  $t=t+1$\;
  }
\caption{Iteration of Prior}
\end{algorithm}

\subsection{Calculation of the Generalization Bound}
\label{appen: single_step_kl}
To obtain the optimal noise covariance of $\textbf{(P2)}$, we first derive the explicit form of the generalization bound $\widetilde{\gen}_T$ with the prior given by Eq. (\ref{eq: update_prior}) as the following lemma:
\begin{lemma}[Calculate $\widetilde{\gen}_T$]
\label{lem: single_step_form_kl_general}
Let Assumption \ref{assum: j_invariant} hold. Let the prior $P$ is given by the update rule Eq. (\ref{eq: update_prior}). Then, the generalization bound $\widetilde{\gen}_{T}$ can be represented as 
\begin{equation*}
    \widetilde{\gen}_{T}=\mathbb{E}_{\bS_{\bJ}}\sqrt{\frac{(a_2-a_1)^2}{2}\sum_{t=1}^T\mathbb{E}_{\bS_{\bJ^c},\bV_{[t-1]}}\mathbb{E}_{P_{s-1}^{\bJ, \bS_{\bJ},\bV_{[s-1]}}}A_t(\bS,\bW_{t-1})},
\end{equation*}
where $A_t(\bS,\bW)$ is given as
\begin{small}
\begin{align*}
   A_t(\bS,\bW) \overset{\triangle}{=}&\frac{1}{2}\left(\sigma_t(\bW)\operatorname{tr}\left(  \bSig_t (\bS,\bW)^{-1}\right)+\ln \left(\operatorname{det} \bSig_t (\bS,\bW)\right)-d \right)
    \\
    &+\frac{\eta_t^2}{2Nb_t}\left(\frac{N}{N-1}\right)^2\operatorname{tr}\left(\bSig_t(\bS)^{-1}\bSig^{sd}_{\bS,\bW_{t-1}}\right)-\frac{1}{2}d\ln \sigma_t(\bW).
\end{align*}
\end{small}

\end{lemma}
\begin{proof}
By the definition of $\widetilde{\gen}_{T}$, we have 
\begin{equation*}
    \widetilde{\gen}_T = \mathbb{E}_{\bS}\sqrt{\frac{(a_2-a_1)^2}{2}\mathbb{E}_{\bV_{[T]},\bJ}\operatorname{KL}\left(P^{\bJ,\bS_{\bJ},\bV_{[T]}}\left\Vert Q^{\bS,\bV_{[T]}}\right.\right)},
\end{equation*}
which by the decomposition of KL divergence (Lemma \ref{lem: decomposition_kl}) further leads to 
\begin{align*}
     \widetilde{\gen}_T
     =&\mathbb{E}_{\bS}\sqrt{\frac{(a_2-a_1)^2}{2}\mathbb{E}_{\bV_{[T]},\bJ}\sum_{s=1}^T \mathbb{E}_{P_{s-1}^{\bJ, \bS_{\bJ},\bV_{[T]}}} \operatorname{KL}\left(P_{s|(s-1)}^{\bJ,\bS_{\bJ},\bV_{[T]}}\left\Vert Q_{s|(s-1)}^{\bS,\bV_{[T]}}\right.\right)}
     \\
     =&\mathbb{E}_{\bS}\sqrt{\frac{(a_2-a_1)^2}{2}\sum_{s=1}^T\mathbb{E}_{\bV_{[s]},\bJ} \mathbb{E}_{P_{s-1}^{\bJ, \bS_{\bJ},\bV_{[s-1]}}} \operatorname{KL}\left(P_{s|(s-1)}^{\bJ,\bS_{\bJ},\bV_{s}}\left\Vert Q_{s|(s-1)}^{\bS,\bV_{s}}\right.\right)}
     \\
     \overset{(*)}{=}&\mathbb{E}_{\bS}\sqrt{\frac{(a_2-a_1)^2}{2}\sum_{s=1}^T\mathbb{E}_{\bV_{[s-1]}} \mathbb{E}_{P_{s-1}^{\bJ, \bS_{\bJ},\bV_{[s-1]}}}\mathbb{E}_{\bV_{s},\bJ} \operatorname{KL}\left(P_{s|(s-1)}^{\bJ,\bS_{\bJ},\bV_{s}}\left\Vert Q_{s|(s-1)}^{\bS,\bV_{s}}\right.\right)}
     \\
     \overset{(**)}{=}&\mathbb{E}_{\bS}\sqrt{\frac{(a_2-a_1)^2}{2}\sum_{s=1}^T\mathbb{E}_{\bV_{[s-1]},\bJ} \mathbb{E}_{P_{s-1}^{\bJ, \bS_{\bJ},\bV_{[s-1]}}}\mathbb{E}_{\bV_{s},\bJ} \operatorname{KL}\left(P_{s|(s-1)}^{\bJ,\bS_{\bJ},\bV_{s}}\left\Vert Q_{s|(s-1)}^{\bS,\bV_{s}}\right.\right)},
\end{align*}
where in Eq. $(*)$ we exchange the order between $\mathbb{E}_{P_{s-1}^{\bJ, \bS_{\bJ},\bV_{[s-1]}}}$ and $\mathbb{E}_{\bV_{s},\bJ}$ due to Assumption \ref{assum: j_invariant}, and Eq. $(**)$ is due to that $\mathbb{E}_{P_{s-1}^{\bJ, \bS_{\bJ},\bV_{[s-1]}}}\mathbb{E}_{\bV_{s},\bJ} \operatorname{KL}\left(P_{s|(s-1)}^{\bJ,\bS_{\bJ},\bV_{s}}\left\Vert Q_{s|(s-1)}^{\bS,\bV_{s}}\right.\right)$ is independent of $\bJ$ by Assumption \ref{assum: j_invariant}.

Therefore, we only need to prove $\mathbb{E}_{\bV_{t},\bJ} \operatorname{KL}\left(P_{t|(t-1)}^{\bJ,\bS_{\bJ},\bV_{t}}\left\Vert Q_{t|(t-1)}^{\bS,\bV_{t}}\right.\right)=A(t)$, which can be obtained by 
 
\begin{align*}
    &\mathbb{E}_{\bV_{t},\bJ} \operatorname{KL}\left(P_{t|(t-1)}^{\bJ,\bS_{\bJ},\bV_{t}}\left\Vert Q_{t|(t-1)}^{\bS,\bV_{t}}\right.\right)
    \\
    &\overset{(\bullet)}{=}\frac{1}{2} \mathbb{E}_{\bV_{t},\bJ} \bigg(\left(\mu^{\bS,\bV_t,\bJ,\bW_{t-1}}\right)^\top\bSig_t(\bS,\bW_{t-1})^{-1}\mu^{\bS,\bV_t,\bJ,\bW_{t-1}}+\ln\frac{\det \bSig_t(\bS,\bW_{t-1})}{\sigma_t(\bW_{t-1})^d}
    \\
    &\;\;+\tr\left(\sigma_t(\bW_{t-1})\bSig_t(\bS,\bW_{t-1})^{-1}\right)\bigg)-\frac{d}{2}
    \\
   &= \frac{1}{2}\mathbb{E}_{\bV_{t},\bJ} \bigg(\tr\left(\bSig_t(\bS,\bW_{t-1})^{-1}\mu^{\bS,\bV_t,\bJ,\bW_{t-1}}\left(\mu^{\bS,\bV_t,\bJ,\bW_{t-1}}\right)^\top\right)+\ln\frac{\det \bSig_t(\bS,\bW_{t-1})}{\sigma_t(\bW_{t-1})^d}
    \\
    &\;\;+\tr\left(\sigma_t(\bW_{t-1})\bSig_t(\bS,\bW_{t-1})^{-1}\right)\bigg)-\frac{d}{2}
    \\
    &= \frac{1}{2} \tr\left(\bSig_t(\bS,\bW_{t-1})^{-1}\mathbb{E}_{\bJ,\bV_t}\mu^{\bS,\bV_t,\bJ,\bW_{t-1}}\left(\mu^{\bS,\bV_t,\bJ,\bW_{t-1}}\right)^\top\right)+\frac{1}{2}\ln\frac{\det \bSig_t(\bS,\bW_{t-1})}{\sigma_t(\bW_{t-1})^d}
    \\
    &\;\;+\frac{1}{2}\tr\left(\sigma_t(\bW_{t-1})\bSig_t(\bS,\bW_{t-1})^{-1}\right)-\frac{d}{2}
    \\
    &\overset{(\circ)}{=}\frac{1}{2}\left(\sigma_t(\bW_{t-1})\operatorname{tr}\left(  \bSig_t (\bS,\bW_{t-1})^{-1}\right)+\ln \left(\operatorname{det} \bSig_t (\bS,\bW_{t-1})\right)-d \right)-\frac{1}{2}d\ln \sigma_t(\bW_{t-1})
    \\
    &\;\;+\frac{\eta_t^2}{2Nb_t}\left(\frac{N}{N-1}\right)^2\operatorname{tr}\left(\bSig_t(\bS,\bW_{t-1})^{-1}\bSig^{sd}_{\bS,\bW_{t-1}}\right),
\end{align*}
where Eq. $(\bullet)$ is due to Lemma \ref{lem: kl_gau}, where $\mu^{\bS,\bV_t,\bJ,\bW_{t-1}}$ is defined by Eq.(\ref{eq:difference_mean}), and Eq. $(\circ)$ is obtained by Lemma \ref{lem: sampling}.

The proof is completed.
\end{proof}

By Lemma \ref{lem: single_step_form_kl_general}, for any $t\in [T]$, $\bS$, and $\bW_{t-1}$, $\gen_{[T]}$ depend on $\bSig_t(\bW,\gen_{[T]})$ only through $A_t(\bS,\bW)$, and the solution of optimizing $A_t$ with respect to $\bSig_t$ under Constraint \ref{constraint: trace} has already been given by Lemma \ref{lem: matrix_opt}. 
We then complete the proof of Theorem \ref{thm: main_theorem} in the next section by combining Lemma \ref{lem: single_step_form_kl_general} and Lemma \ref{lem: matrix_opt} together.

\subsection{Proof of Theorem \ref{thm: main_theorem}}
\label{sec: proof_main_theorem}
In this section, we first restate Theorem \ref{thm: main_theorem} with explicit form of $\tilde{\omega}_i^{\bS,\bW}$ (omitted in the main text). We then provide the proof of the theorem by Lemma \ref{lem: single_step_form_kl_general} and Lemma \ref{lem: matrix_opt}.
\begin{theorem}
\label{thm:main_general}
Let prior and posterior be defined  as Eq.(\ref{eq: update_prior}) and Eq.(\ref{eq:state-dependent SGLD}), respectively. Then, with Assumption \ref{assum: j_invariant}, the solution of $\textbf{(P2)}$ is given by
\begin{equation*}
    \bSig^{*}_t(\bS,\bW)= \boldsymbol{Q}_{\bS,\bW}^{sd} \Diag (\tilde{\omega}^{\bS,\bW}_{t,1},\cdots,\tilde{\omega}^{\bS,\bW}_{t,d})\left(\boldsymbol{Q}_{\bS,\bW}^{sd}\right)^\top,
\end{equation*}
where 
\begin{equation*}
    \tilde{\omega}^{\bS,\bW}_{t,i}=\frac{\sqrt{1-4\lambda^* \left(\mathbb{E}_{\bJ}\sigma_t(\bS_{\bJ}, \bW)+\frac{\eta_t^2}{Nb_t}\left(\frac{N}{N-1}\right)^2\omega^{\bS,\bW}_i\right)}-1}{-2\lambda^*},
\end{equation*}
$\lambda^*$ is determined by $  \sum_{i=1}^d\tilde{\omega}^{\bS,\bW}_{t,i}=c_t\left(\bS,\bW\right)$, and
$\boldsymbol{Q}_{\bS,\bW}^{sd}$ is the orthogonal matrix that diagonalizes $\bSig^{sd}_{\bS,\bW}$ as
\begin{equation*}
    \bSig^{sd}_{\bS,\bW}=\boldsymbol{Q}_{\bS,\bW}^{sd}\Diag (\omega^{\bS,\bW}_{1},\cdots,\omega^{\bS,\bW}_{d})\left(\boldsymbol{Q}_{\bS,\bW}^{sd}\right)^\top. 
\end{equation*}
\end{theorem}
\begin{proof}[Proof of Theorem \ref{thm: main_theorem}]
By  Lemma \ref{lem: single_step_form_kl_general}, $\gen_T$ depends on $\bSig_t(\bS,\bW)$ only through $A_t(\bS,\bW)$, and we have
\begin{align*}
    &\bSig^{*}_s(\bS,\bW)
    \\
    &=\arg\min_{\text{Constraint } \ref{constraint: trace}}A_s(\bS,\bW)
    \\
    &=\arg\min_{\tr(\bSig)=c_s(\bS,\bW)}\frac{1}{2}\bigg(\operatorname{tr}\left(  \bSig^{-1}\left( \sigma_s(\bS_{\bJ}, \bW_{s-1})\mathbb{I}+\frac{\eta_s^2}{Nb_s}\left(\frac{N}{N-1}\right)^2\bSig^{sd}_{\bS,\bW_{s-1}}\right)\right)
    \\
    &\;\;-d  \ln\sigma_s(\bS_{\bJ}, \bW_{s-1})-d +\ln \left(\operatorname{det} \bSig\right)\bigg)
    \\
    &=\arg\min_{\tr(\bSig)=c_s(\bS,\bW)}\frac{1}{2}\bigg(\operatorname{tr}\left(  \bSig^{-1}\left( \sigma_s(\bS_{\bJ}, \bW_{s-1})\mathbb{I}+\frac{\eta_s^2}{Nb_s}\left(\frac{N}{N-1}\right)^2\bSig^{sd}_{\bS,\bW_{s-1}}\right)\right) 
    \\
    &\;\;+\ln \left(\operatorname{det} \bSig\right)\bigg).
\end{align*}
Applying Lemma \ref{lem: matrix_opt} completes the proof.

\end{proof}

\subsection{Smaller Condition Number}
 \label{appen: condition_number}
In this section, we demonstrate why the optimal noise of Theorem \ref{thm: main_theorem} has smaller condition number than $\bSig^{sd}$ as the following corollary. 

\begin{corollary}
 The optimal noise covariance $\bSig^*$ given by Theorem \ref{thm: main_theorem} has smaller condition number than $\bSig^{sd}$.
\end{corollary}

\begin{proof}
We prove this claim following two steps.

Firstly, the noise covariance of the prior is isotropic, has condition number $1$, and push the condition number of  $\sigma_t\mathbb{I}+\frac{\eta_t^2}{Nb_t}\left(\frac{N}{N-1}\right)^2\Sigma^{sd}_{\bS,\bW}$ smaller than $\Sigma^{sd}_{\bS,\bW}$.

Secondly, the optimal solution $G$ of Lemma \ref{lem: matrix_opt} always has a smaller condition number than $B$, which implies that $\Sigma_t^*(\bS,\bW)$ has smaller condition number than $B=\sigma_t\mathbb{I}+\frac{\eta_t^2}{Nb_t}\left(\frac{N}{N-1}\right)^2\Sigma^{sd}_{\bS,\bW}$. Hence the condition number of $\Sigma_t^*(\bS,\bW)$ is smaller than $\Sigma^{sd}_{\bS,\bW}$.
\end{proof}

\section{Experiments}
\label{appen: experiment}
In this section, we introduce the settings of the experiments in Fig. (\ref{fig:compare-with-sgd}) Fig.(\ref{fig:train_all}), Fig.(\ref{fig: test}), and Fig.(\ref{fig:ratio}). We further include an additional experiment comparing the generalization error between SGLD with square rooted empirical gradient covariance (SREC-SGLD) (the optimal noise covariance in Theorem \ref{thm: greedy}) and SGLD with empirical gradient covariance (EC-SGLD) subject to Constraint \ref{constraint: trace}. 

\subsection{Experiment settings}
\label{appen:setting}
For both Fig. (\ref{fig:compare-with-sgd}), Fig.(\ref{fig:train_all}), Fig.~(\ref{fig: test}), and Fig.~(\ref{fig:ratio}), we adopt the same setting as the Fashion-MNIST experiment of \cite[Section D.3]{zhu2018anisotropic} despite enlarging the training set. Specifically, we use the $4$-layer convolutional neural network as our model to conduct multi-class classification on Fashion-MNIST \cite{xiao2017fashion}. Concretely, this convolutional neural network can be expressed in order as: convolutional layer with $10$ channel and filter size $5\times5$, max-pool layer with kernel size $2\time 2$ and stride $2$,
convolutional layer with $10$ channel and filter size $5\times 5$, max-pool with kernel size $2\time 2$, two fully connected layer with width $50$. Our training set consists of 10,000 examples uniformly sampled without replacement from the Fashion-MNIST dataset. Our training set is larger than that in \cite{zhu2018anisotropic} (which only contains 1200 samples), but is still one sixth of the whole Fashion-MNIST dataset due to the computational burden of gradient descent (without mini-batch) in the SGLD. The learning rates of all SGLD are set to $0.07$, which is exactly the same as \cite{zhu2018anisotropic}. We also set the learning rate of SGD in Fig. (\ref{fig:compare-with-sgd}) to $0.07$ for fair comparison with SGLD.

\textbf{Empirical gradient covariance:}
We use top $100$ eigenvalues to approximate the empirical gradient covariance matrix. Specifically, we decompose the matrix $\bSig^{sd}_{\bS,\bW}$ into $(Q_{\bS,\bW})^\top (\omega^{\bS,\bW}_{[d]}) Q_{\bS,\bW}$, and use $(Q_{\bS,\bW})^\top (\omega^{\bS,\bW}_{[100]},\boldsymbol{0}_{d-100}) Q_{\bS,\bW}$ to approximate $\bSig^{sd}_{\bS,\bW}$.

\textbf{Noise Scale:} In Fig. (\ref{fig:compare-with-sgd}) and Fig.(\ref{fig:train_all}), the traces of all SGLDs are set to be $\tr(\bSig^{sd}_{\bS,\bW})$; in Fig.
(\ref{fig: test}), the traces are set to be $ \tr((\bSig_{\bS,\bW}^{sd})^{1/2})$ and $5 \tr((\bSig_{\bS,\bW}^{sd})^{1/2})$, respectively in (a) and (b); in Fig. (\ref{fig:ratio}), the traces are set to be $ \tr((\bSig_{\bS,\bW}^{sd})^{1/2})$.

\textbf{Noise frequency:}
Similar to \cite{zhu2018anisotropic}, we re-estimate the noise structure
of all SGLDs every 10 epochs to ease
the computational burden.

\subsection{Comparison between EC-SGLD and SREC-SGLD}
We further conduct an experiment to compare the generalization performance between Iso-SGLD, EC-SGLD and SREC-SGLD, with the traces of the covariance are all set to be $5 \tr((\bSig_{\bS,\bW}^{sd})^{1/2})$, and all other settings consistent with Appendix \ref{appen:setting}. The generalization error along the iteration of SREC-SGLD, Iso-SGLD, and EC-SGLD is plotted as Fig. \ref{fig:nosqrt}, where one can easily observe the generalization error of SREC-SGLD is the smallest, which supports Theorem \ref{thm: greedy}.
\begin{figure}[ht]
    \centering
    \includegraphics[scale=0.3]{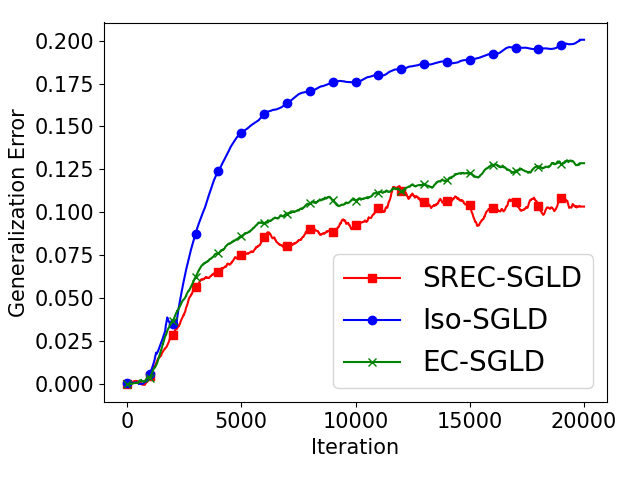}
    \caption{Comparison of generalization error for SGLDs with different noise structures. Traces of the covariances are all set to be $5 \tr((\bSig_{\bS,\bW}^{sd})^{1/2})$.}
    \label{fig:nosqrt}
\end{figure}